\documentclass[twoside,11pt]{article}
\usepackage{jmlr2e}
\usepackage{algorithmic}
\usepackage{algorithm}
\usepackage{amsmath}
\ShortHeadings{On Tail Decay Rate Estimation of Loss Function Distributions }{Haxholli and Lorenzi}
\firstpageno{1}

\begin{document}
\title{On Tail Decay Rate Estimation of Loss Function Distributions }
\author{\name Etrit Haxholli \email etrit.haxholli@inria.fr\\
       \addr Inria\\
       Univesity Côte d'Azur\\
       2004 Rte des Lucioles, 06902 Valbonne, France
       \AND
       \name Marco Lorenzi \email marco.lorenzi@inria.fr \\
       \addr Inria\\
       Univesity Côte d'Azur\\
       2004 Rte des Lucioles, 06902 Valbonne, France}

\editor{}

\maketitle

\begin{abstract}
% Distributions of loss functions are central objects of study, as they represent the performance of machine learning models. 

%For a given model and machine learning problem, the true distribution of the loss function is generally not known, as we only possess a finite number of sampled data points. 

%For a given machine learning problem, the distribution of a model's loss function is generally not known, as it can be only assessed on a finite number of sampled data points. 

The study of loss function distributions is critical to characterize a model's behaviour on a given machine learning problem. For example, while the quality of a model is commonly determined by the average loss assessed on a testing set, this quantity does not reflect the existence of the true mean of the loss distribution. Indeed, the finiteness of the statistical moments of the loss distribution is related to the thickness of its tails, which are generally unknown.

%It is therefore critical of studying the behaviour and decay rate of tails of loss function distributions.

%One of the key statistics of such distribution to be inferred is the mean, which can be estimated via different existing methods. Even though the estimations through such methods will always be a finite value, in reality, the first moment is not guaranteed to exist, and its existence cannot be established through such approaches. From a theoretical standpoint, the highest existing moment of a distribution is strongly linked to the thickness of its tails, which highlights the importance of studying the behaviour and decay rate of tails of loss function distributions.

%The estimation of tails parameters from sampling loss distributions is challenging. 
%Since different training/testing partitions determine different conditional loss distributions over points in the sample space, the total loss distribution must be recovered by integrating over the space of training sets. In this case, as we show in this work, the finiteness of the sample approximations negatively affects the reliability and efficiency of classical tail estimation methods from the Extreme Value Theory, such as the Peaks-Over-Threshold approach.
Since typical cross-validation schemes determine a family of testing loss distributions conditioned on the training samples, the total loss distribution must be recovered by marginalizing over the space of training sets.
As we show in this work, the finiteness of the sampling procedure negatively affects the reliability and efficiency of classical tail estimation methods from the Extreme Value Theory, such as the Peaks-Over-Threshold approach.
In this work we tackle this issue by developing a novel general theory for estimating the tails of marginal distributions, when there exists a large variability between locations of the individual conditional distributions underlying the marginal. To this end, we demonstrate that under some regularity conditions, the shape parameter of the marginal distribution is the maximum tail shape parameter of the family of conditional distributions. We term this estimation approach as \emph{cross-tail estimation (CTE)}. 

We test cross-tail estimation in a series of experiments on simulated and real data\footnote{The code is available at https://github.com/ehaxholli/CTE}, showing the improved robustness and quality of tail estimation as compared to classical approaches, and providing evidence for the relationship between overfitting and loss distribution tail thickness.  


\end{abstract}

\section{Introduction}

Loss function distributions form critical subjects of analysis, serving as barometers for machine learning model performance. In the context of a particular model and associated machine learning task, the authentic distribution of the loss function is typically elusive; we predominantly have access to a finite sample set, borne from diverse choices of training and testing sets. To facilitate performance comparisons across different models based on the underlying loss function distributions, a spectrum of methodologies has been established. Traditional strategies derive from information criteria such as the Akaike Information Criterion (AIC) \cite{akaike1973information, 1100705}, an asymptotic approximation of the Kullback-Leibler divergence between the true data distribution and the fitting candidate, and its corrected version (AICc) \cite{doi:10.1080/03610927808827599, 10.1093/biomet/76.2.297}, in addition to the Bayesian Information Criterion (BIC) \cite{10.1214/aos/1176344136}. The application of these information criteria, especially the AIC, is often constrained by the multiple inherent approximations and assumptions \cite{burnham2007model}, making them less feasible in certain scenarios. However, it warrants mention that more recent penalized criteria have considerably expanded their suitability for realistic setups \cite{10.1214/aos/1176324452, JMLR:v10:arlot09a}. Simultaneously, other methodologies, termed splitting/resampling methods, have been devised, wherein a subset of the data is deployed to assess the performance of the trained model. This group of methodologies is expansive, predicated on a diverse range of partitioning and evaluation tactics devised to address data heterogeneity and imbalance \cite{10.2307/2342192, cochran2007sampling}.

In the domain of cross-validation strategies, the common metric employed for gauging model performance is the sample mean of the loss function distribution. This practice, though invariably providing a finite numerical value, does not assure the existence of the first statistical moment or those of higher order.
Moreover, this metric, in spite of its prevalence, should not be construed as a sole indicator of the model's performance robustness or reliability, as it potentially overlooks the nuances and intricacies inherent to the underlying data distribution and model architecture. While it is true that MSE (or AIC) allow to rank models according to their relative performance on a given dataset, these scores still have limited
value in quantifying the overall stability of a model. From a theoretical perspective, there is a strong correlation between the uppermost existing moment of a distribution and the thickness of its tail. This underscores the significance of examining the behavioural traits and decay rate of the tails of loss function distributions.

In order to proceed, we first must be able to model the tails of distributions and to quantify their "thickness". Extreme Value Theory (EVT) is an established field concerned with modelling the tails of distributions. One of the fundamental results in EVT is the Pickands–Balkema–De Haan Theorem, which states that the tails of a large class of distributions can be approximated with generalized Pareto ones \cite{10.1214/aos/1176343003, de2007extreme}. In practice, the shape and scale parameter of the generalized Pareto are approximated from a finite sample, while its location parameter is always zero. It is the shape parameter which quantifies tail thickness, with larger values corresponding to heavier tails. The resulting estimation method is called Peaks-Over-Threshold (POT).

In the context of distributions of loss functions, for each training set, there is a corresponding conditional loss function distribution over points in the sample space. The actual total loss function distribution, the entity of our interest, is the weighted sum (integral) of all such conditional distributions, that is, it is the distribution created after marginalizing across the space of training datasets. In practice, we have a finite number of conditional distributions, as we have a finite number of training sets. Furthermore, for each of these conditional distributions, we only possess an approximation of them, derived from the samples in the testing set. The empirical approximation of the total loss function distribution therefore consists of the union of the sample sets of conditional distributions. Within this setting, the estimation of the tail shape of the total loss function distribution could be ideally carried out by applying POT on this union of samples.

In theory, as we show in this work, the role of the thickest conditional tails in determining the decay rate of the marginal is preserved, since the marginal and conditional distributions are defined everywhere, which allows the assessment of tails at extreme locations. Unfortunately, in practice, the finiteness of the sampling affects the estimation of the tail of the marginal distribution, as the tails may be poorly or not even represented across different conditional distributions. To be more specific, during marginalization, samples from the tails of heavy tailed distributions can be overshadowed by the samples from the non-tail part of individual thin tailed ones. This suggests that modelling the tails of a marginal distribution by the usual application of POT can give inaccurate results in practice.

In this paper, we develop a general method to mitigate the issue of estimating the tails of marginal distributions, when there exists a large variability between locations of the individual conditional distributions underlying the marginal. The proposed solution enables a reduction in the sample size requirements, in the experiments we conducted. To this end, we demonstrate that under some regularity conditions, the shape parameter of the marginal distribution is precisely the maximum tail shape parameter of the family of conditional distributions. We refer to the method constructed from this result as \emph{cross tail estimation}, due to similarities that it shares with Monte Carlo cross validation. 
Furthermore, we show evidence of polynomial decay of tails of distributions of model predictions, and empirically demonstrate a relationship between the thickness of such tails and overfitting. An additional benefit of using the approach proposed here instead of the standard POT, is the reduced computational time in the case that the marginal is estimated from many conditional distributions.

The following is a summary of the structure of the paper:  In Section 2 we recall some of the main concepts and results from Extreme Value Theory. In Section 3, we state and generalize the main problem, which we tackle in Section 4, by  building our theory. We conclude Section 4, by proving three statements which are useful for the experimental part, and by highlighting the relation between the tail of a distribution and its moments.  In the final section, we show experimentally that our method can improve estimation in practice, as compared to the standard use of POT.

\section{Related Work and Background}

This section initially provides a succinct overview of Monte Carlo cross validation, given its conceptual similarities with the proposed method, 'cross tail estimation'. The subsequent subsection outlines standard results and definitions from extreme value analysis, forming the foundational bedrock for the proofs presented in Section \ref{theory}.

\subsection{Monte Carlo Cross Validation}
Let $D=\{(x_1,y_1),...(x_n,y_n)\}$, be a set of data samples drawn form the same distribution. During each iteration $i$ we sample $k$ samples $D_i=\{(x_{\pi(1)},y_{\pi(1)}),...,(x_{\pi(k)},y_{\pi(k)})\}$ without repetition from the original dataset $D$, and consider it as the training set for that iteration. The set $D\setminus D_i$ is then used as the testing set. The quantity of interest, during iteration $i$ is the sample mean of the loss of the model trained on $D_i$, namely $\hat{f}_{D_i}$, over the points of the testing set:
\begin{equation}\label{crossval1}
   \tilde{M}^L_i:=\frac{1}{|D\setminus D_i|} \sum\limits_{j\in D\setminus D_i} L(\hat{f}_{D_i}(x_j),y_j),
\end{equation}
for a given loss function $L$.
\newline
We evaluate the total performance of the model, based on its average performance over different choices of the training/testing sets, that is, the true evaluation metric is:
\begin{equation}\label{crossval2}
   \tilde{M}^L:=\frac{1}{m} \sum\limits_{i \in [m]} \tilde{M}^L_i=\frac{1}{m} \sum\limits_{i \in [m]} \frac{1}{|D\setminus D_i|} \sum\limits_{j\in D\setminus D_i} L(\hat{f}_{D_i}(x_j),y_j),
\end{equation} where $m$ is the number of iterations.
\newline
 A detailed discussion on cross validation, elucidating its similarities with our proposed method for tail estimation in marginal loss function distributions, namely 'cross tail estimation', is presented in Subsection \ref{paralcrossval}.

\subsection{Extreme Value Theory}

Extreme value theory (EVT) or extreme value analysis (EVA) is a branch of statistics dealing with the extreme deviations from the median of probability distributions. Extreme value theory is closely related to failure analysis and dates back to 1923, when Richard von Mises discovered that the Gumbell distribution is the limiting distribution of the maximum of an iid sequence, sampled from a Gaussian distribution. In 1928, Ronald A. Fisher and Leonard H. C. Tippett  in \cite{1928PCPS...24..180F}, characterized the only three possible non-degenerate limiting distributions of the maximum in the general case: Frechet, Gumbel and Weibull. In 1943, Boris V. Gnedenko, gave a rigorous proof of this fact in \cite{10.2307/1968974}. This result is known  Fisher–Tippett–Gnedenko theorem, and forms the foundation of EVT.  The three aforementioned limiting distributions of the maximum can be written in compact form and they are known as the class of extreme value distributions:
\newline
\begin{definition}\label{def1}The Generalized Extreme Value Distribution is defined as follows:
\begin{equation}\label{gevd}
    G_{\xi,a,b}(x)=e^{-(1+\xi (ax+b))^{-\frac{1}{\xi}}}, \ \ \  1+\xi(a x+b)>0,
\end{equation}
where $b \in \mathbb{R}$, $\xi \in \mathbb{R}\setminus \{0\}$ and  $a>0$. For $\xi=0$, we define the generalized Extreme Value Distribution as the limit when $\xi \rightarrow 0$, that is 
\begin{equation}\label{gumbell}
    G_{0,a,b}(x)=e^{-e^{-ax-b}}.
\end{equation}
\end{definition}

\begin{theorem}[Fisher–Tippett–Gnedenko]\label{ftg}: Let $X$ be a real random variable with distribution $F_X$. Denote by $\{X_1, X_2,...,X_n\}$ a set of iid samples from the distribution $F_X$, and define $M_n=\max\{X_1,...,X_n\}$. If there exist two sequences $\{c_i>0\}_{i \in \mathbb{N}}$ and $\{d_i \in \mathbb{R}\}_{i \in \mathbb{N}}$, such that
\begin{equation}\label{2.2}
c_n^{-1}(M_n-d_n)\xrightarrow{d} F \text{ as } n\rightarrow\infty,
\end{equation}
for some non-degenerate distribution $F$, then we must have $F(x)=G_{\xi,a,b}(x)$, for some $b, \xi \in \mathbb{R}, a>0$.
\end{theorem}
If $X$ is a random variable as in Theorem \ref{ftg}, such that $F(x)=G_{\xi,a,b}(x)$, we say that $F_X$ is in the Maximum Domain of Attraction of $G_{\xi,a,b}(x)$, and we write $F_X \in MDA(\xi)$. Depending on whether $\xi>0$, $\xi=0$, $\xi<0$, we say that $F_X$ is in the MDA of a Frechet, Gumbell, or Weibull distribution respectively. 
\newline
\begin{definition} A Generalized Pareto distribution with location parameter zero is defined as below:
\begin{align}\label{d1_1}
    G_{\xi,\sigma}(w)=
    \begin{cases}
      1-(1+\xi \frac{w}{\sigma}))^{-\frac{1}{\xi}} & \text{for $\xi\neq 0$}
      \\
      1-e^{-\frac{w}{\sigma}} & \text{for $\xi=0$}
    \end{cases}  ,
\end{align}
where $w >0$ when $\xi>0$ and $0<w<-\frac{\sigma}{\xi}$ for $\xi<0$. The shape parameter is denoted by $\xi$, while the scale parameter by $\sigma$.
\end{definition}
 \cite{10.1214/aos/1176343003}, and \cite{10.1214/aop/1176996548} proved that the limiting distribution of samples larger than a threshold is a Generalized Pareto distribution, whose location parameter is zero. 
\begin{theorem}[Pickands–Balkema–De Haan]: Let $X$ be a random variable with distribution $F_X$ and $x_F\leq\infty$ such that $\forall x>x_F,\  \bar{F}_X(x)=0$. Then $F_X\in MDA(\xi) \iff \exists g: (0,\infty)\rightarrow(0,\infty)$ such that 
\begin{equation}\label{pickands}
\lim_{u \rightarrow x_F}\sup_{y\in[0,x_F-u]}|\bar{F}^X_u(y)-\bar{G}_{\xi,g(u)}(y)|=0,
\end{equation} where $\bar{F}^X_u(y)=\frac{1-F_X(y+u)}{1-F_X(u)}.$
\end{theorem}
This result forms the basis of the well-known Peak-Over-Threshold (POT) method which is used in practice to model the tails of distributions. 
The shape parameter can be estimated via different estimators such as the Pickands Estimator or the Deckers-Einmahl-de Haan Estimator (DEdH), \cite{10.1214/aos/1176347397}.

\begin{definition} Let $X_1$, $X_2$,...,$X_n$ be iid samples from the distribution $F_X$. If we denote with $X_{1,n}$, $X_{2,n}$,..., $X_{n,n}$ the samples sorted in descending order, then the Pickands estimator is defined as follows:
\begin{equation}\label{picks_estimator}
\hat{\xi}^{(P)}_{k,n}=\frac{1}{\ln{2}}\ln{\frac{X_{k,n}-X_{2k,n}}{X_{2k,n}-X_{4k,n}}}.
\end{equation}

\end{definition} 

\begin{definition}\label{dedh_estimator_def} Let $X_1$, $X_2$,...,$X_n$ be iid samples from the distribution $F_X$. If we denote with $X_{1,n}$, $X_{2,n}$,..., $X_{n,n}$ the samples sorted in descending order, then the DEdH estimator is defined as follows:
\begin{equation}\label{dedh_estimator}
\hat{\xi}^{(H)}_{k,n}=1+H_{k,n}^{(1)}+\frac{1}{2}\left(\frac{(H_{k,n}^{(1)})^2}{H_{k,n}^{(2)}} -1\right)^{-1},
\end{equation}
where
\begin{equation}
H^{(1)}_{k,n}=\frac{1}{k}\sum_{j=1}^k (\ln{X_{j,n}}-\ln{X_{{k+1},n}})
\end{equation}
and 
\begin{equation}
H^{(2)}_{k,n}=\frac{1}{k}\sum_{j=1}^k (\ln{X_{j,n}}-\ln{X_{{k+1},n}})^2.
\end{equation}
\end{definition}
An important result which we are going to use frequently in our proofs is Theorem \ref{mda_rv}, which can be found in \cite{embrechts2013modelling, de2007extreme}, and gives the connection between the maximum domain of attraction and slowly varying functions. 
\begin{definition} A positive measurable function L is called slowly varying if it is defined in some neighborhood of infinity and if:
\begin{equation}
    \lim_{x\rightarrow \infty}\frac{L(ax)}{L(x)}=1, \text{for all } a>0.
\end{equation}
\end{definition}
\begin{theorem}[Representation Theorem, see \cite{10.2307/2038824}]:
A positive measurable function L on $[x_0,\infty]$ is slowly varying if and only if it can be written in the form:
\begin{equation}\label{slovar_rep}
L(x)=e^{c(x)}e^{\int_{x_0}^x \frac{u(t)}{t} dt},
\end{equation} where $c(t)$ and $u(t)$, are measurable bounded functions such that $\lim_{x\rightarrow\infty}c(t)=c_0\in (0,\infty)$ and $u(t)\rightarrow0$ as $t\rightarrow\infty$.
\end{theorem}

\begin{proposition}\label{prop1}\cite{mikosch1999regular} If $L$ is slowly varying then for every $\epsilon>0$: 
\begin{equation}\label{slvarpoly}
\lim_{x \rightarrow \infty}x^{- \epsilon} L(x)=0.
\end{equation}
\end{proposition} 
\begin{proof}
We give a proof in the Appendix for the sake of completeness.
\end{proof}

\begin{theorem}\label{mda_rv}:
If $X\in MDA(\xi)$ and $x_F$ is such that $\forall x>x_F,\  \bar{F}_X(x)=0$ then:
\begin{itemize}
  \item $\xi>0 \iff \bar{F}_X(x)=x^{-\frac{1}{\xi}}L(x)$, where L is slowly varying,
  \item $\xi<0 \iff \bar{F}_X(x_F-\frac{1}{x})=x^{\frac{1}{\xi}}L(x)$, where L is slowly varying,\
  \item $\xi=0 \iff \bar{F}_X(x)=c(x)e^{-\int_w^x \frac{1}{a(t)}dt},\ w<x<x_F\leq\infty$, where $c$ is a measurable function satisfying $c(x)\rightarrow c>0$ as $x\uparrow x_F$, and $a(x)$ is a positive, absolutely continuous function (with respect to Lebesgue measure) with density $a'(x)$ having $\lim_{x\uparrow x_F}a'(x)=0$. If $x_F<\infty$ then $\lim_{x\uparrow x_F}a(x)=0$ as well.
\end{itemize}
\end{theorem}

\section{Setup and Problem Statement}
In the first part of this section, we give an example which illustrates why naively using the POT method can give unsatisfactory results. In the second part, we formalize the problem of tail modelling of total loss distributions and show that such modelling is prone to the weakness described in Subsection 1. In the third subsection, we introduce Cross-Tail-Estimation (CTE) which alleviates these issues, and we demonstrate an analogy between CTE and Cross-Validation. In the last subsection, we prepare the settings for Section \ref{theory}, where we provide theoretical justifications (Theorem \ref{otheorem2}) for the use of CTE.

\subsection{Preamble}

Estimating the tails of marginal distributions via standard methods such as using POT directly, can give unsatisfactory results. In order to get a glimpse of the issue, let's assume that our variable of interest is $X>0$, which in turns depends on the variable $Z$. For simplicity we can assume that $Z$ can be either 0 or 1, with equal probability, and if $Z=0$ then $f(x|Z=0)$ is a thick tailed distribution whose even first moment does not exist, while if $Z=1$ then $f(x|Z=1)$ is a Gaussian distribution, with a large mean. 
It is known that the tail shape parameter of $f(x)=\sum_{i=1}^n p(Z=i)f(x|Z=i)$ is determined by the conditional distribution $f(x|Z=i)$ with the thickest tail. In our case, $n$ above is 2, and the tail of $f(x)$ is defined by the fat tail of $f(x|Z=0)$.
Suppose we proceed with the standard POT approach, that is, we integrate out the random variable $Z$, and subsequently estimate the shape parameter of the tail of $f(x)$. In practice, when the number of samples is limited, it is possible that none of the samples of $X$ from the fat tailed distributions exceeds those of the Gaussian due to the difference between their locations. Therefore, the sample tail of the marginal (mixture) distribution $f(x)=\frac{1}{2}(f(x|Z=0)+f(x|Z=1))$ is defined by the sample tail of the Gaussian $f(x|Z=1)$, while in reality, as mentioned, the tail of $f(x)$ is defined by the fat tail of $f(x|Z=0)$. Of course, in the ideal case where the sampling process is not finite, we would recover the true tail shape; however, for practical applications, estimating the shape tails parameters of $f(x|Z=0)$ and $f(x|Z=1)$ separately can be necessary.
\newline
In some settings, random variable $Z$ might have a continuous distribution, e.g. $Range(Z)=\mathbb{R}$, instead of $Range(Z)=\{1,2,..,n\}$. Such is the case presented in the next subsection, where $f(x)=\int p(z)f(x|z) dz$ . A natural question that arises in this scenario is whether the tail of the marginal $f(x)$ is still determined by the largest tail of the conditional distributions $f(x|z)$. As we will prove in Section 4, under some regularity conditions, the answer is in the affirmative. 

\subsection{Problem Statement}\label{probstat}

We assume that each data sample $\boldsymbol{(X,Y)}$ comes from distribution $\mathcal{D}$ and that the sampling is independent. We have used the symbol $\boldsymbol{X}$ to denote the features and the symbol $\boldsymbol{Y}$ to denote the labels. The training set will be defined as a random vector comprised of iid random vectors $\boldsymbol{(X,Y)}$ sampled from $\mathcal{D}$. More precisely, after fixing a natural number $k$, we define a training set as $\boldsymbol{V}=[\boldsymbol{(X,Y)}_1, \boldsymbol{(X,Y)}_2,..., \boldsymbol{(X,Y)}_k]$, where each $\boldsymbol{(X,Y)}_i$ has distribution $\mathcal{D}$. On the other hand, a test point naturally is defined as a sample from $\mathcal{D}$, i.e., $\boldsymbol{U}=\boldsymbol{(X,Y)}$. In practice, the realisation of $\boldsymbol{U}$ should not be an entry in $\boldsymbol{V}$.
\newline
 A model which is trained on $\boldsymbol{V}$ to predict $\boldsymbol{Y}$ from $\boldsymbol{X}$ is denoted as $\boldsymbol{\hat{h}_V(X)}$. The prediction error on the testing datum $\boldsymbol{U}$ of a model trained on $\boldsymbol{V}$ is denoted as $W_{\boldsymbol{V}}\boldsymbol{(U)}$. For the remainder of the paper we assume that $W_{\boldsymbol{V}}\boldsymbol{(U)}>0$ and notice that the probability density function of $W_{\boldsymbol{V}}\boldsymbol{(U)}$ is
\begin{equation}\label{eq.3.1}
f_W(w)=\int f_{W,{\boldsymbol{V}}}(w,{\boldsymbol{v}}) d{\boldsymbol{v}}= \int f_{\boldsymbol{V}}({\boldsymbol{v}}) f(w|{\boldsymbol{V}}={\boldsymbol{v}}) d{\boldsymbol{v}}=\int f_{\boldsymbol{V}}({\boldsymbol{v}}) f_{\boldsymbol{v}}(w) d{\boldsymbol{v}},
\end{equation}
therefore the distribution function of $W_{\boldsymbol{V}}\boldsymbol{(U)}$ is:
\begin{equation}\label{eq.3.2}
F_W(w)=\int f_{\boldsymbol{V}}({\boldsymbol{v}}) F_{\boldsymbol{v}}(w) d{\boldsymbol{v}}.
\end{equation}
$F_{\boldsymbol{v}}(w)$ is the distribution of the prediction error (loss) of the model trained on training set $\boldsymbol{v}$, while $F(w)$ is the unconditional distribution of the loss. Our goal is to estimate the shape of the tails of $F_W(w)$, by estimating the shape of the tails of the distributions $F_{\boldsymbol{v}}(w)$ conditioned on the training sets $\boldsymbol{v}$. 

\subsection{Cross Tail Estimation}\label{paralcrossval}

We will denote with $\xi_{\boldsymbol{v}}$ the tail shape parameter of $F_{\boldsymbol{v}}(w)$ and with $\xi$ the shape tail parameter of $F_W(w)$. Our goal in Section 4 is to prove that under some regularity conditions if $\exists \boldsymbol{v}$, such that $\xi_{\boldsymbol{v}}>0$, then $\xi=\max\{\xi_{\boldsymbol{v}}|\boldsymbol{v}\}$, and if $\forall \xi_{\boldsymbol{v}}\leq 0$, then we have $\xi\leq 0$.
\begin{algorithm}
\caption{Naive Cross Tail Estimation}\label{algo1}
\begin{algorithmic}
\STATE {\bfseries Data:} $D=[\boldsymbol{(x,y)}_1, \boldsymbol{(x,y)}_2,..., \boldsymbol{(x,y)}_n]$
\STATE {\bfseries Define:} $A=\{\}$
\STATE {\bfseries Fix the number of training sets (rounds):} $m \in \mathbb{N}$
\REPEAT
\STATE 1. sample $\boldsymbol{(x,y)}_{\pi(1)}$, ..., $\boldsymbol{(x,y)}_{\pi(k)}$ from $\boldsymbol{(x,y)}_1, \boldsymbol{(x,y)}_2,..., \boldsymbol{(x,y)}_n$
\STATE 2. train model $\boldsymbol{\hat{h}_v}$ on $\boldsymbol{v}=[\boldsymbol{(x,y)}_{\pi(1)}$, ..., $\boldsymbol{(x,y)}_{\pi(k)}]$
\STATE 3. calculate the prediction errors $W_{\boldsymbol{v}}\boldsymbol{(U)}$ of model $\boldsymbol{\hat{h}_v}$ on the testing set $D\setminus\boldsymbol{v}$
\STATE 4. group the calculated prediction errors in the set $E_{\boldsymbol{v}}(D)$
\STATE 5. apply the Pickands or DEdH estimator on $E_{\boldsymbol{v}}(D)$ to estimate $\xi_{\boldsymbol{v}}$
\STATE 6. add $\hat{\xi}_{\boldsymbol{v}}$ to $A$
\UNTIL $|A|=m$
\RETURN $\max{A}$ if $\max{A}>0$, else return 'non-positive'
\end{algorithmic}
\end{algorithm} 
This motivates Algorithm \ref{algo1} which we name 'Naive Cross Tail Estimation' (NCTE). Since for each $\boldsymbol{v}$, the estimated $\hat{\xi}_{\boldsymbol{v}}$ is prone to estimation errors, taking the maximum $\hat{\xi}_{\boldsymbol{v}}$ over all ${\boldsymbol{v}}$ tends to cause NCTE to overestimate the true $\xi$, especially when the number of conditional distributions $F_{\boldsymbol{v}}(w)$ is large. For this reason we also present Algorithm \ref{algo2},  named 'Cross Tail Estimation' (CTE), where we split the samples from $F_{\boldsymbol{v}}(w)$ into $p$ sets in order to get $p$ estimates of the tail shape parameter of $F_{\boldsymbol{v}}(w)$, that is $\{\hat{\xi}^1_{\boldsymbol{v}},\hat{\xi}^2_{\boldsymbol{v}},...,\hat{\xi}^p_{\boldsymbol{v}}\}$.
Our final estimation of $\xi_{\boldsymbol{v}}$ is the average of the $p$ estimations, i.e., $\frac{1}{p}\sum_{i=0}^p \hat{\xi}^i_{\boldsymbol{v}}$.  A more detailed justification for utilizing Algorithm \ref{algo2} is given in Appendix E.
\begin{algorithm}
\caption{Cross Tail Estimation}\label{algo2}
\begin{algorithmic}
\STATE {\bfseries Data:} $D=[\boldsymbol{(x,y)}_1, \boldsymbol{(x,y)}_2,..., \boldsymbol{(x,y)}_n]$
\STATE {\bfseries Define:} $A=\{\}$
\STATE {\bfseries Fix the number of training sets (rounds):} $m \in \mathbb{N}$
\REPEAT
\STATE 1. sample $\boldsymbol{(x,y)}_{\pi(1)}$, ..., $\boldsymbol{(x,y)}_{\pi(k)}$ from $\boldsymbol{(x,y)}_1, \boldsymbol{(x,y)}_2,..., \boldsymbol{(x,y)}_n$
\STATE 2. train model $\boldsymbol{\hat{h}_v}$ on $\boldsymbol{v}=[\boldsymbol{(x,y)}_{\pi(1)}$, ..., $\boldsymbol{(x,y)}_{\pi(k)}]$
\STATE 3. calculate the prediction errors $W_{\boldsymbol{v}}\boldsymbol{(U)}$ of model $\boldsymbol{\hat{h}_v}$ on the testing set $D\setminus\boldsymbol{v}$
\STATE 4. group the calculated prediction errors in the set $E_{\boldsymbol{v}}(D)$
\STATE 5. split $E_{\boldsymbol{v}}(D)$ into $\{E^1_{\boldsymbol{v}}(D),...,E^p_{\boldsymbol{v}}(D)\}$
\STATE 6. apply the Pickands or DEdH estimator on each $E^i_{\boldsymbol{v}}(D)$ to get an estimate $\hat{\xi}^i_{\boldsymbol{v}}$ of $\xi_{\boldsymbol{v}}$
\STATE 7. average over $p$ to get the final estimate $\hat{\xi}_{\boldsymbol{v}}=\frac{1}{p}\sum_{i=0}^p \hat{\xi}^i_{\boldsymbol{v}}$ of $\xi_{\boldsymbol{v}}$
\STATE 8. add $\hat{\xi}_{\boldsymbol{v}}$ to $A$
\UNTIL $|A|=m$
\RETURN $\max{A}$ if $\max{A}>0$, else return 'non-positive'
\end{algorithmic}
\end{algorithm}
We notice that Algorithm \ref{algo2} is identical to Algorithm \ref{algo1} when $p=1$. 
\newline \emph{Remark}: Estimating particular statistics of $F_W(w)$ through the statistics of  $F_{\boldsymbol{v}}(w)$ as in in Algorithm \ref{algo1} and \ref{algo2} is a key component of Cross Validation.
 During cross validation,  a training set ${\boldsymbol{v}}$ and a testing set  $D\setminus\boldsymbol{v}$ are selected in each iteration, during which the following conditional expectation is then estimated:
\begin{equation}\label{cond_expectation}
\begin{split}
 \mathbb{E}[W_{\boldsymbol{V}}(\boldsymbol{U})|{\boldsymbol{V}}={\boldsymbol{v}}]= \int w f_{\boldsymbol{v}}(w)dw.
\end{split}
\end{equation}
The estimates of $\mathbb{E}[W_{\boldsymbol{V}}(\boldsymbol{U})|{\boldsymbol{V}}]$ received in each iteration are then averaged to get an estimation of the total expectation:
\begin{equation}\label{tot_expectation}
\begin{split}
\mathbb{E}_{{\boldsymbol{U,V}}}(W_{\boldsymbol{V}}({\boldsymbol{U}}))=\int wf(w)dw =\int f_{\boldsymbol{V}}({\boldsymbol{v}}) \int w f_{\boldsymbol{v}}(w)dw d{\boldsymbol{v}}=\\
=\int f_{\boldsymbol{V}}({\boldsymbol{v}}) \mathbb{E}[W_{{\boldsymbol{V}}}(U)|{\boldsymbol{V}}={\boldsymbol{v}}] d{\boldsymbol{v}}=\mathbb{E}[\mathbb{E}[W_{{\boldsymbol{V}}}({\boldsymbol{U}})|{\boldsymbol{V}}={\boldsymbol{v}}]].
\end{split}
\end{equation}
In the language of Section \ref{probstat}, the mean of distribution $F_W(w)$ is the average of the means of the conditional distributions $F_{\boldsymbol{v}}(w)$.
\newline
This statement about sums stands parallel with our claim about extremes; that the shape parameter of the tail of $F_W(w)$, if positive, is the maximum of the shape parameters of the tails of the conditional distributions $F_{\boldsymbol{v}}(w)$.

\subsection{The General Problem}

Generalizing the problem stated in Section 3.2 requires considering a one dimensional random variable of interest $X$, dependent on other random variables $\{Z_1, Z_2,..., Z_n\}$, such that the probability density function of $X$ is
\begin{equation}\label{3.4}
\begin{split}
f_{X}(x) =\int f(z_1,...,z_n, x) d_{z_1}\cdots d_{z_n}
\end{split}
\end{equation}
\begin{equation}\label{3.4.2}
\begin{split}
=\int f(\boldsymbol{z})f(x|\boldsymbol{z}) d\boldsymbol{z}=\int f(\boldsymbol{z})f_{\boldsymbol{z}}(x) d\boldsymbol{z}.
\end{split}
\end{equation}
Integrating with respect to $x$ we get 
\begin{equation}\label{3.5}
\begin{split}
F_{X}(x)=\int f(\boldsymbol{z}) F(x|\boldsymbol{z}) d\boldsymbol{z}=\int f(\boldsymbol{z})F_{\boldsymbol{z}}(x) d\boldsymbol{z}.
\end{split}
\end{equation}
In this case, with regards to the previous section, we notice that $\boldsymbol{Z}=\boldsymbol{V}$ is the training set on which we condition, while $X=W$ is the random variable of interest.
 In Section \ref{theory}, we give several results which relate the tails of $F_{X}(x)$ and $F(x|\boldsymbol{z})$, culminating with Theorem \ref{otheorem2} which justifies the usage of the CTE algorithm, by providing limiting behaviour guarantees. 

\section{Theoretical Results}\label{theory}
In this section, we build our theory of modelling the tails of marginal distributions, which culminates with Theorem \ref{otheorem2}. We conclude this section by proving three statements which are useful in the experimental Section \ref{experiments}, and give the relation between the existence of the moments of a distribution and the thickness of its tails. Unless stated otherwise, the proofs of all the statements are given in Appendix A.

\subsection{Tails of marginal distributions}

For two given distributions, whose tails have positive shape parameters, we expect the one with larger tail parameter to decay slower. Indeed:
\begin{lemma}\label{lemma1} If $F_1 \in MDA(\xi_1)$ and $F_2 \in MDA(\xi_2)$, and if $\xi_1>\xi_2>0$, then $\lim_{x \rightarrow \infty} \frac{\bar{F}_2(x)}{\bar{F}_1(x)}=0$.
 \end{lemma}

In a similar fashion, regardless of the signs of the shape parameters, we expect the one with larger tail parameter to decay slower. In fact we have the following:
\begin{lemma}\label{lemma2}
If $F_1 \in MDA(\xi_1)$ and $F_2 \in MDA(\xi_2)$ then:
\begin{enumerate}
  \item  If $\xi_1>0$ and $\xi_2=0$ then $\lim_{x \rightarrow \infty} \frac{\bar{F}_2(x)}{\bar{F}_1(x)}=0$. 
  \item If $\xi_1=0, x_{F_1}=\infty$ and $\xi_2<0$ then $\lim_{x \rightarrow \infty} \frac{\bar{F}_2(x)}{\bar{F}_1(x)}=0$.
  \item  If $\xi_1>0$ and $\xi_2<0$ then $\lim_{x \rightarrow \infty} \frac{\bar{F}_2(x)}{\bar{F}_1(x)}=0$. 
\end{enumerate}
\end{lemma}  

Despite the fact that a linear combination of slowly varying functions is not necessarily slowly varying, the following statement holds true:
\begin{lemma}\label{lemma3}
If for $i\in \{1,...,n\}$ we let $L_i(x)$ be slowly varying functions, and $\{a_1,...,a_n\}$ be a set of positive real numbers, then 
$$
L(x)=\sum\limits_{i=1}^n a_i L_i(x)
$$
 is slowly varying.
\end{lemma}

In the case of a mixture of a finite number of distributions the following known result holds:
\begin{theorem}\label{otheorem1}
Let $Z:\Omega\rightarrow A\subset\mathbb{R}^n$ be a random vector where $|A|<\infty$. At each point $\boldsymbol{z}_1,..,\boldsymbol{z}_n\in A$, we define a distribution $F_{\boldsymbol{z}_i}(x)\in MDA(\xi_{i})$ and assume that $\xi_{\max}:=\max(\xi_1=\xi_{z_1},...,\xi_n=\xi_{z_n})>0$. If the set $\{p_1,...,p_n\}$ is a set of convex combination parameters, that is $\sum\limits_i p_i=1$ and $p_i>0$ then:
\begin{equation}\label{t1_1}
F(x)=\sum\limits_i^n p_i F_{{\boldsymbol{z}}_i}(x)\in MDA(\xi_{\max}).
\end{equation}
If $\xi_{\max}\leq0$ then if $\xi_F$ exists we have $\xi_F\leq0$.
\end{theorem}  

\begin{proof} 
While this result is well known, we give an alternative proof in Appendix A, using the Pickands-Balkema-De Haan Theorem.
\end{proof}

\emph{From now on, we assume that the functions $F_A(x)=\int_A f_{\boldsymbol{Z}}({\boldsymbol{z}})F_{\boldsymbol{z}}(x)d{\boldsymbol{z}}$ defined on any element $A$  of the Borel $\sigma-algebra$ induced by the usual metric are in the $MDA$ of some extreme value distribution. Furthermore, we assume that the pdf $ f_{\boldsymbol{Z}}({\boldsymbol{z}})$ is strictly positive everywhere in its domain.}
\newline

Proposition \ref{prop1} states that every slowly varying function is sub-polynomial. That is for any $\delta>0$ and any slowly varying function $L(x)$, if we are given any $\gamma>0$, then we can find $x(L,\delta,\gamma)>0$, such that for all $x>x(L,\delta,\gamma)$, the inequality $x^{-\delta}L(x)<\gamma $ holds. However, since $x(L,\delta,\gamma)$ depends on the function $L$, assuming that we have a family of $\{L_{\boldsymbol{z}}| {\boldsymbol{z}}\in A\}$, where $A$ is a measurable set, the set $\{x(L_{\boldsymbol{z}},\delta,\gamma)|{\boldsymbol{z}} \in A\}$ can be unbounded, suggesting that the beginning of the tail of $\bar{F}_{\boldsymbol{z}}(x)=x^{-\frac{1}{\xi_{\boldsymbol{z}}}}L_{\boldsymbol{z}}(x)$ can be postponed indefinitely across the family $\{F_{\boldsymbol{z}}|{\boldsymbol{z}} \in A\}$. These concepts are formalized in the following:

\begin{definition} For a set $A$, the family of sub-polynomial functions $\{L_{\boldsymbol{z}}(x)| {\boldsymbol{z}}\in A\}$ is called $\gamma$-uniformly sub-polynomial if for any fixed $\delta>0$, there exists a $\gamma(\delta)$ so that the set $\{x_0|{\boldsymbol{z}} \in A\}$ is bounded from above, where $x_0=x_0(L_{\boldsymbol{z}},\delta,\gamma)$  is the smallest value for which when $x>x_0$ we have $x^{-\delta}L_{\boldsymbol{z}}(x)<\gamma $.
\end{definition}

\begin{proposition}\label{propUniSubPol} Let ${\boldsymbol{Z}}:\Omega\rightarrow A\subset\mathbb{R}^n$ be a random vector where $A$ is measurable and define a family of sub-polynomial functions $\{L_{\boldsymbol{z}}(x)| {\boldsymbol{z}}\in A\}$, which we assume is $\gamma$-uniformly sub-polynomial. Then for a probability density function $f_{\boldsymbol{Z}}({\boldsymbol{z}})$ on $A$ induced by ${\boldsymbol{Z}}$, the function $L(x)=\int_A f_{\boldsymbol{Z}}({\boldsymbol{z}}) L_{\boldsymbol{z}}(x) d{\boldsymbol{z}}$  is sub-polynomial.
\end{proposition}  

In the following theorem, we assume that all conditional distributions have positive tail shape parameters, and we show that the marginal distribution cannot have a tail shape parameter larger (smaller) than the largest (smallest) tail shape parameter across conditional distributions. Furthermore, if the tail shape parameters vary continuously across the space of conditional distributions, then the tail shape parameter of the marginal is precisely the same as the maximal tail shape parameter of the conditional distributions.
\begin{theorem}\label{otheorem2a} Let ${\boldsymbol{Z}}:\Omega\rightarrow A\subset\mathbb{R}^n$ be a random vector where $A$ is measurable. At each point ${\boldsymbol{z}}\in A$ define a distribution $F_{\boldsymbol{z}}(x)\in MDA(\xi_{\boldsymbol{z}})$, and suppose there exist $\xi_{lo},\ \xi_{up}$ such that $\forall {\boldsymbol{z}}\in A,\  0<\xi_{lo}\leq \xi_{\boldsymbol{z}} \leq \xi_{up}$. Furthermore, let $L_{\boldsymbol{z}}(x)$ be the slowly varying function corresponding to $F_{\boldsymbol{z}}(x)$. If the family $\{L_{\boldsymbol{z}}(x)|{\boldsymbol{z}}\in A\}$ is $\gamma$-uniformly sub-polynomial, then for $F(x)=\int_A f_{\boldsymbol{Z}}({\boldsymbol{z}})F_{\boldsymbol{z}}(x)d{\boldsymbol{z}}$ we have $\xi_{lo}\leq\xi_F\leq\xi_{up}$. Furthermore, if $\xi_{\boldsymbol{z}}$ is continuous in ${\boldsymbol{z}}$, then $\xi_F=\xi_{\max}$, where $\xi_{\max}:=\sup\{\xi_{\boldsymbol{z}}|{\boldsymbol{z}} \in A\}$.
\end{theorem}  

Similarly to the case when $F_{\boldsymbol{z}}(x)$ are in the $MDA(\xi_{\boldsymbol{z}})$ for $\xi_{\boldsymbol{z}}>0$, if we wish to extend the results above, regularity conditions are required for the $\xi_{\boldsymbol{z}}\leq 0$ case. We notice that if $F_z(x) \in MDA(\xi)$ for $\xi \leq 0$, then $\bar{F}_{\boldsymbol{z}}(x)$ itself is sub-polynomial, whether its support is bounded or not. This observation motivates the following:

\begin{definition} For a set $A$, define the family of distribution functions $\mathcal{F_A}=\{F_{\boldsymbol{z}}(x)| {\boldsymbol{z}}\in A\}$, and define $A^+=\{{\boldsymbol{z}}|\xi_{\boldsymbol{z}}>0\},\ A^{-}=\{{\boldsymbol{z}}|\xi_{\boldsymbol{z}}\leq 0\}$. We say family $\mathcal{F_A}$ has stable cross-tail variability if,
\begin{itemize}
	\item $\{L_{\boldsymbol{z}}(x)| {\boldsymbol{z}}\in A^+\}$ is $\gamma$-uniformly sub-polynomial,
	\item $\{\bar{F}_{\boldsymbol{z}}(x)| {\boldsymbol{z}}\in A^-\}$ is $\gamma$-uniformly sub-polynomial.
\end{itemize}

\end{definition}
We notice that in the previous theorem, if for all ${\boldsymbol{z}}$ we have $0<\xi_{\boldsymbol{z}}\leq \epsilon$, then $\xi_F \leq \epsilon$. If the corresponding family $\mathcal{F_A}=\{F_{\boldsymbol{z}}(x)| {\boldsymbol{z}}\in A\}$ has stable cross-tail variability, this holds independently from the lower bound of $\{\xi_{\boldsymbol{z}}|{\boldsymbol{z}} \in A\}$. Indeed:

\begin{lemma}\label{lemma5}
Let ${\boldsymbol{Z}}:\Omega\rightarrow A\subset\mathbb{R}^n$ be a random vector where $A$ is measurable. At each point ${\boldsymbol{z}}\in A$ define a distribution $F_{\boldsymbol{z}}(x)\in MDA(\xi_{\boldsymbol{z}})$, and suppose that $\forall {\boldsymbol{z}}\in A,\ \xi_{\boldsymbol{z}} \leq \epsilon$. If the family $\{F_{\boldsymbol{z}}(x)|{\boldsymbol{z}}\in A\}$ has stable cross-tail variability, then for $F(x)=\int_A f_{\boldsymbol{Z}}({\boldsymbol{z}})F_{\boldsymbol{z}}(x)d{\boldsymbol{z}}$ we have $\xi_F\leq \epsilon$.
\end{lemma}

\begin{corollary}\label{corollarynonpositive}Let ${\boldsymbol{Z}}:\Omega\rightarrow A\subset\mathbb{R}^n$ be a random vector where $A$ is measurable. At each point ${\boldsymbol{z}}\in A$ define a distribution $F_{\boldsymbol{z}}(x)\in MDA(\xi_{\boldsymbol{z}})$, and suppose that $\forall {\boldsymbol{z}}\in A,\ \xi_{\boldsymbol{z}} \leq 0$. If the family $\{F_{\boldsymbol{z}}(x)|{\boldsymbol{z}}\in A\}$ has stable cross-tail variability, then for $F(x)=\int_A f_{\boldsymbol{Z}}({\boldsymbol{z}})F_{\boldsymbol{z}}(x)d{\boldsymbol{z}}$ we have $\xi_F\leq 0$.
\end{corollary}  
\begin{proof} 
We notice that for any $\epsilon>0$, we have $\xi_{\boldsymbol{z}}<\epsilon$ for all ${\boldsymbol{z}} \in A$. Hence, from the previous Lemma we conclude that $\xi_F\leq \epsilon, \forall\epsilon>0$.
\end{proof}
Finally, we prove the generalization of Theorem \ref{otheorem2a} in the case that the tail shape parameters $\xi_{\boldsymbol{Z}}$ of the conditional distributions are real numbers:

\begin{theorem}\label{otheorem2}
Let ${\boldsymbol{Z}}:\Omega\rightarrow A\subset\mathbb{R}^n$ be a random vector where $A$ is measurable. At each point ${\boldsymbol{z}}\in A$ define a distribution $F_{\boldsymbol{z}}(x)\in MDA(\xi_{\boldsymbol{z}})$, where $\xi_{\boldsymbol{z}}$ is continuous and $\xi_{\max}>0$. If the family $\{F_{\boldsymbol{z}}(x)|{\boldsymbol{z}}\in A\}$ has stable cross-tail variability, then for $F(x)=\int_A f_{\boldsymbol{Z}}({\boldsymbol{z}})F_{\boldsymbol{z}}(x)d{\boldsymbol{z}}$ we have $\xi_F=\xi_{\max}$. In the case that $\xi_{\max} \leq0$  then $\xi_F \leq0$.
\end{theorem}

Examples when the conditions of Theorem \ref{otheorem2} hold, as well as when they are violated, can be found in Appendix C and B, respectively.

\subsection{Useful propositions for the experimental part}
In this subsection, we prove three statements which are useful in the experimental Section \ref{experiments}, and state the well-known relation between the existence of the moments of a distribution and the thickness of its tails.
\begin{proposition}\label{proposition1}
Let $F_X$ be the distribution of the random variable $X$. We define $X_1$ to be a random variable whose distribution is the normalized right tail of $F_X$, that is:
\begin{align}\label{cut_tail1}
    F_{X_1}(x)=
    \begin{cases}
      0 & \text{for $x\leq0$}
      \\
      \frac{F(x)-F(0)}{1-F(0)}&\text{for $x>0$}
    \end{cases}  .
\end{align}
Similarly we define $X_2$ whose distribution is the normalized left tail of $F_X$,
\begin{align}\label{cut_tail2}
    F_{X_2}(x)=
    \begin{cases}
      0 & \text{for $x<0$}
      \\
      \frac{F(0)-F(-x)}{F(0)}&\text{for $x\geq0$}
    \end{cases}  .
\end{align}
If $F_{X_1}\in MDA(\xi_1)$, $F_{X_2}\in MDA(\xi_2)$, and $\max\{\xi_1,\xi_2\}>0$, then: $$\xi_{|X|}=\max\{\xi_1,\xi_2\}.$$
If $F_{X_1}\in MDA(\xi_1)$, $F_{X_2}\in MDA(\xi_2)$, and $\max\{\xi_1,\xi_2\}\leq0$, then: $$\xi_{|X|}\leq0.$$
\end{proposition}
\begin{proof}
Since
\begin{equation}\label{p1_1}
\begin{split}
F_{|X|}(x)=\mathbb{P}(|X|<x)=\mathbb{P}(X<x|X>0)\mathbb{P}(X>0)+\mathbb{P}(-X<x|X\leq0)\mathbb{P}(X\leq0)
\\
=p_1 F_{X_1}(x)+p_2 F_{X_2}(x),
\end{split}
\end{equation}
Theorem \ref{otheorem1} gives the desired conclusion.
\end{proof}

\begin{proposition}\label{proposition2}
Let $X$ be a random variable such that $X\in MDA(\xi_X>0)$. If we define $Y$ to be equal to $X^\alpha$, for some $\alpha\in \mathbb{R^+}$, then $Y \in MDA(\xi_Y)$ where $\xi_Y=\alpha \xi_X$. If $\xi_X\leq 0$ then $\xi_Y\leq 0$.
\end{proposition}

It is important to notice that we can estimate the shape of the tail of $W_{\boldsymbol{V}}\boldsymbol{(U)}$ by also conditioning on the test label $\boldsymbol{y}$:
\begin{equation}\label{eq.3.3}
f_W(w)=\int f_{W,\boldsymbol{Y}}(w,\boldsymbol{y}) d\boldsymbol{y}= \int f_{\boldsymbol{Y}}(\boldsymbol{y}) f(w|\boldsymbol{Y}=\boldsymbol{y}) d\boldsymbol{y}=\int f_{\boldsymbol{Y}}(\boldsymbol{y}) f_{\boldsymbol{y}}(w) d\boldsymbol{y}
\end{equation}
\begin{equation}\label{eq.3.4a}
F_W(w)=\int f_{\boldsymbol{Y}}(\boldsymbol{y}) F_{\boldsymbol{y}}(w) d\boldsymbol{y}.
\end{equation}
We use this fact to prove the following:
\begin{proposition}\label{proposition3}
Let the loss function be defined as $W_{\boldsymbol{V}}({\boldsymbol{U}})=|{Y}-\hat{f}_{\boldsymbol{V}}({\boldsymbol{X}})|^p$ for some $p\in \mathbb{R^+}$, and let $F_y(t)$ be the distribution of $\hat{f}_{\boldsymbol{V}}({\boldsymbol{X}})$ given $Y$. If we assume that the distribution of the labels $Y$ has bounded support $S$, that the family $\{F_y(t)| y \in S\}$ has stable cross-tail variability, and that the shape parameters $\xi_y$ of $F_y(t)$ change continuously, then the tail shape parameters of $W_{\boldsymbol{V}}({\boldsymbol{U}})$ and $|\hat{f}_{\boldsymbol{V}}({\boldsymbol{X}})|^p$ share the same sign, and are identical if either of them is positive.
\end{proposition}

There exists a strong connection between the Maximum Domain of Attraction of a distribution, and the existence of its moments (see \cite{embrechts2013modelling}):
\begin{proposition}\label{prop2} If $F_{|X|}$ is the distribution function of a random variable $|X|$, and $F_{|X|} \in MDA(\xi)$ then:
\begin{equation}\label{4.2.1}
\text{i) if } \xi>0,\text{ then }  \mathbb{E}[|X|^r]=\infty, \forall r\in (\frac{1}{\xi},\infty),
\end{equation}
\begin{equation}\label{4.2.2}
\text{ii) if } \xi \leq 0,\text{ then }  \mathbb{E}[|X|^r]<\infty, \forall r\in (0,\infty).
\end{equation}
\end{proposition}

 This means that, for a model with a positive loss function whose distribution has a shape parameter that is bigger than one, even the first moment of that loss function distribution does not exist. Hence, we would expect that our model has an infinite mean, which would suggest that this model should be eliminated during model ranking. However, if all models possess an infinite mean, it is not advisable to disregard models with smaller medians.
 \newline
\newline
 In Proposition \ref{proposition3}, we showed that if we condition on the testing set, under some assumptions, we can estimate the shape of the total loss distribution, that is the distribution of $W_{{\boldsymbol{V}}}({\boldsymbol{U}})$, by simply investigating the models prediction, without the need for target data. This can also be motivated from the moments of $W_{{\boldsymbol{V}}}({\boldsymbol{U}})$ as shown in Appendix D.

\section{Experiments}\label{experiments}

In this section, we demonstrate the significance of Theorem \ref{otheorem2}. In the first subsection, we show experimental evidence that the estimated shape parameter of the marginal distribution, under the assumption that we have an abundance of sample points, coincides with the maximal shape parameter of individual conditional distributions. In the second subsection, we show that when the sample size is finite, as it is the case in the real world, the method proposed by Theorem \ref{otheorem2} (cross tail estimation) can be necessary to reduce the required sample size for proper tail shape parameter estimation of marginal distributions. Furthermore, in the third subsection, we compare the standard POT and cross tail estimation on real data. For the considered regression scenarios, we notice that when these shape parameters are calculated by cross tail estimation, the magnitude of shape parameters of the distribution of model predictions increases significantly when the model overfits. We also notice that such a relationship does not appear in the case that we use directly the POT method to estimate the aforementioned shape parameters. Finally, in the fourth subsection, we discuss the computational advantages of using cross tail estimation.

\subsection{Validity of Cross Tail Estimation in Practice}\label{s5.1}
The main problem that we tried to tackle in the previous section, was estimating the shape parameters of the tail of distribution $F(x)$:
\begin{equation}\label{5.1}
\begin{split}
F(x)=\int f({\boldsymbol{z}})F_{\boldsymbol{z}}(x) d{\boldsymbol{z}},
\end{split}
\end{equation}
via tail shape estimation of the conditional distributions $F_{\boldsymbol{z}}(x)$. In what follows, we give two experiments showing that this is feasible in practice.
\subsubsection{Experimental Setting}\label{experimental_setting}

For simplicity, we set ${\boldsymbol{z}}$ to be one dimensional, and thus denote the conditional distributions $F_{\boldsymbol{z}}$ as $F_z$, where $z\in \mathbb{R}$. In this case equation (\ref{5.1}) becomes
\begin{equation}\label{5.2}
\begin{split}
F(x)=\int f(z)F_{z}(x) d{z}.
\end{split}
\end{equation}
\newline
First, we define $f(z)$ as a mixture of Gaussian distributions. To do so we choose a mean $\mu_i$ from a uniform distribution in $[-5,5]$ and then a standard deviation $\sigma_i$ from a uniform distribution between $[0,4]$, and together they define a Gaussian distribution $g_i(z)$. We repeat this process for $30$ Gaussian distributions and define $f(z)=\sum_{i=1}^{30} \frac{g_i(z)}{30}$.
\newline
\newline
Second, we define the function ${\xi_z}$ as
\begin{equation}\label{xi_z_imple}
{\xi_z}=\frac{\frac{(nz+2m^2+kz^3)e^{-|z|}+a}{b}+c}{d},
\end{equation}
where $n=1$, $m=2$, $k=2$, $b=5.76$, $a=-3b-3.80$, $d=(\frac{7}{8}\xi_{\max}+\frac{29}{8})^{-1}$ and $c=d\xi_{\max}+3$. The $\xi_{\max}$ in the variables $c, d$  determines the maximum value that the function ${\xi_z}$ takes as long as $\xi_{\max} \in[-4,5]$. More details about the function ${\xi_z}$ are provided in Appendix G. 
\newline
\newline
Third, we define $F_z(x)$ as a generalized Generalized Pareto if  ${\xi_z}\leq0$, otherwise we define it  as  $F_z(x)=1-x^{-\frac{1}{{\xi_z}}}$.
\newline
\newline
The choice of $\xi_{\max}$ completely determines each ${\xi_z}$ and hence each $F_z(x)$, thus it fully defines $F(x)$ in Equation \ref{5.2}. In our experiments the parameter $\xi_{\max}$ takes the following $45$ values $\{-4,-4+0.2,-4+0.4,...,5\}$, that is ${\xi_{j}}=-4+\frac{2j}{10}$, where $j\in\{0,...45\}$. Each choice of $j$ defines a particular $F_j(x)$ on the left side of Equation \ref{5.2}. Also since the maximum ${\xi_{j}}$ determines $\xi_z$ then we denote $\xi_z$ as $\xi_{z,j}$. For each $j$ we repeat the following process $p$ times:

\begin{enumerate}
  \item Define an empty List J and repeat $M$ times the steps (a), (b), (c).
  \begin{enumerate}
  \item Sample a $z$ from the $f(z)$ defined above
  \item For that $z$ calculate $\xi_{z,j}$ (given that $\xi_{\max}={\xi_{j}}$)
  \item For the given  $\xi_{z,j}$ sample a point $x$ from a Generalized Pareto if  $\xi_{z,j}\leq 0$, otherwise sample from $F_z(x)=1-x^{-\frac{1}{\xi_{z,j}}}$. Add this sample to List J.
  \end{enumerate}
    \item Use the Pickands or DEdH estimator on these $M$ samples in List J to estimate the shape parameter of $F_j(x)$. According to Theorem \ref{otheorem2} this estimated value ${\hat{\xi}^k}_{j}$ should be precisely ${\xi_{j}}$.
\end{enumerate}
As guided by the ideas laid in Appendix E, our final estimation of ${\xi_{j}}$ after $p$ repetitions of the process above is ${\hat{\xi}}_{j}=\frac{1}{p}\sum_{k=1}^{p} {\hat{\xi}^k}_{j}$. 
\newline
\newline
In the next subsections we show the results of performing this experiment for each $j$ using the Pickands and the DEdH estimators. 

\subsubsection{Cross tail estimation using the Pickands estimator}
 In this subsection, we show the results of the experiment described in Subsection \ref{experimental_setting}, when the Pickands Estimator is applied. 
\begin{figure}[ht]
\centering
\includegraphics[scale=0.27]{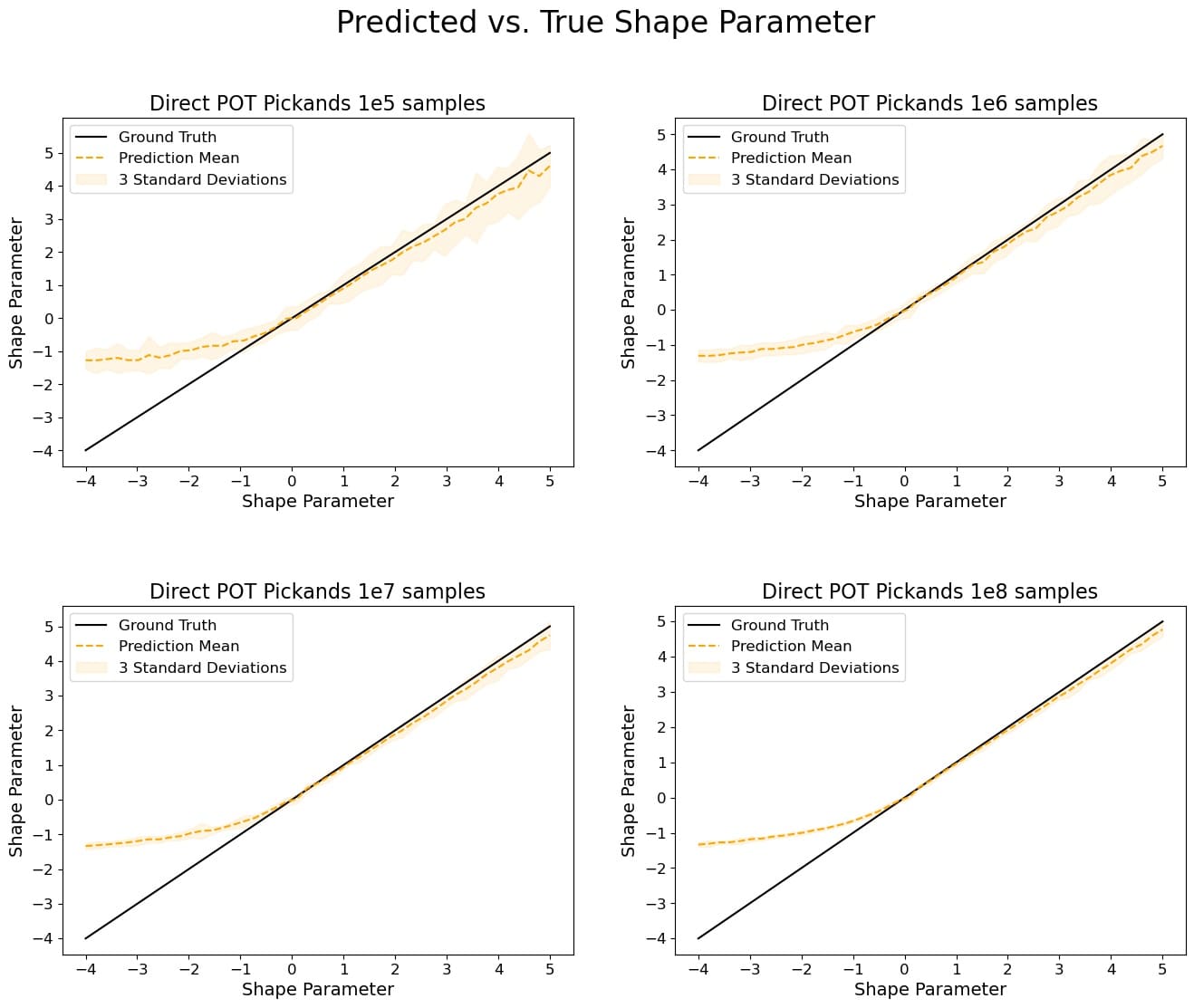}
\caption{ In cases where the maximum tail shape parameter in the mixture of conditional distributions is positive, the estimated shape parameter of the marginal is equal to this maximal value. If this maximum value is negative, the estimated shape parameter is negative. We utilized the Pickands estimator. }
\label{fig5.1.1}
\end{figure}
The results are shown in Figure \ref{fig5.1.1}, where the number $M$ defined in the previous subsection takes the following values $\{10^{5},10^{6},10^{7},10^{8}\}$ and we set $p=10$ . We have executed the experiment $10$ times, and to account for variability across the different runs, we have computed the mean and standard deviation of the results.

\subsubsection{Cross tail estimation using the DEdH estimator}

 In this subsection, we present the results of the experiment described in Subsection \ref{experimental_setting}, when the DEdH Estimator is employed. The results are illustrated in Figure \ref{fig5.1.2}, where the number $M$ defined in the previous subsection takes the values $\{10^{5},10^{6},10^{7},10^{8}\}$ and set $p=10$ . We have executed the experiment 10 times, and to account for variability across the different runs, we have computed the mean and standard deviation of the results.
 
\begin{figure}[ht]
\centering
\includegraphics[scale=0.27]{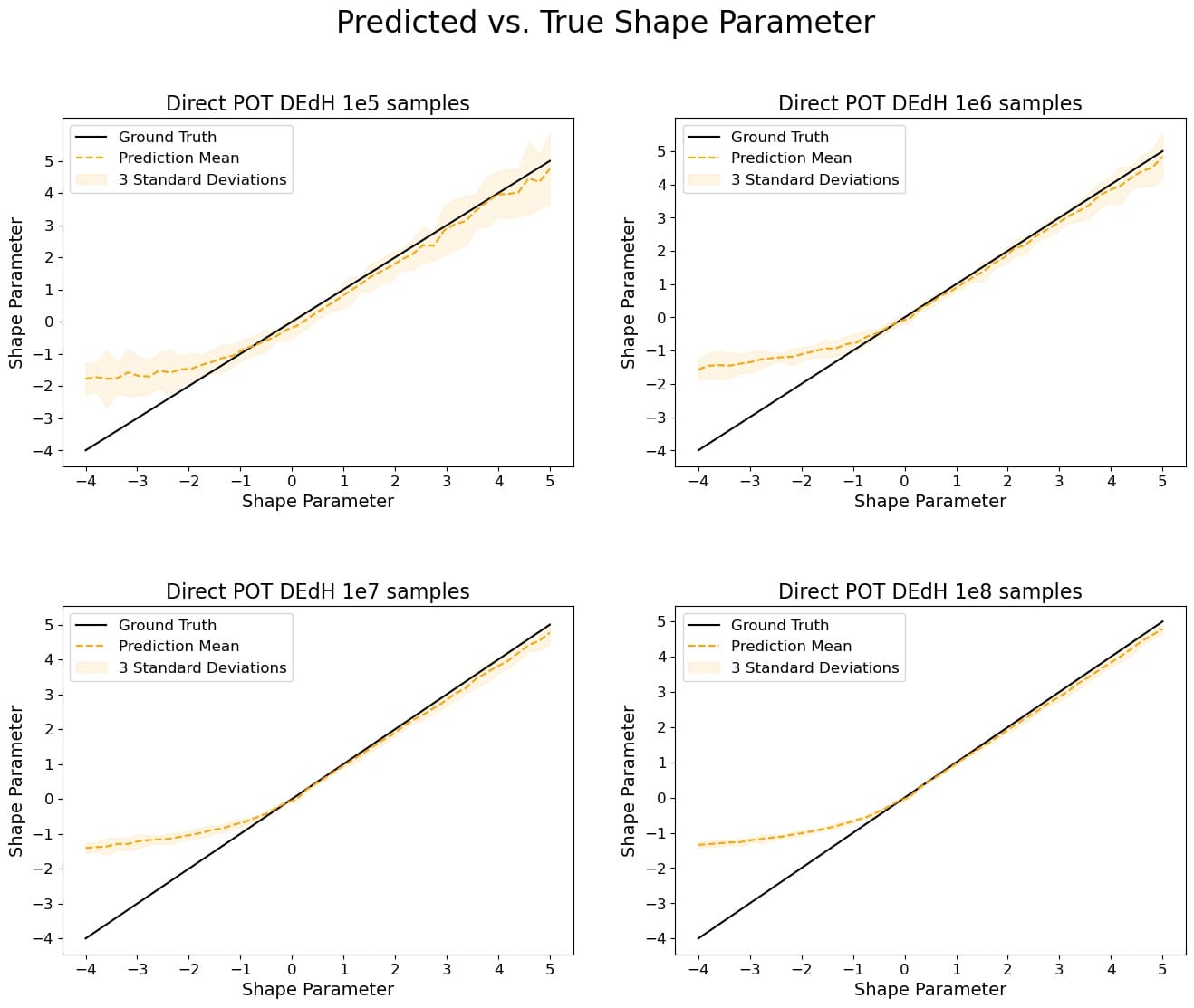}
\caption{In cases where the maximum tail shape parameter in the mixture of conditional distributions is positive, the estimated shape parameter of the marginal is also positive and equal to this maximal value. However, if this maximum value is negative, the estimated shape parameter is also negative. We utilize the DEdH estimator as our estimator of choice.}
\label{fig5.1.2}
\end{figure}

\subsection{Addressing High Variance in the Location of Conditional Distributions: The Necessity of Cross Tail Estimation (CTE)}\label{s5.2}

In subsection \ref{s5.1}, we presented empirical evidence to substantiate Theorem \ref{otheorem2}. Notably, for computational expediency, we elected to set all conditional distributions with a location parameter of zero. This decision was motivated by the fact that, if location parameters were permitted to exhibit significant variability, the direct Peaks Over Threshold (POT) approach would necessitate an unfeasibly large sample size to verify our claims. This issue is addressed in the current subsection, wherein we illustrate that the Conditional Tail Expectation (CTE) approach provides a suitable remedy. Specifically, in subsection \ref{s5.2.1}, we outline modifications to the experimental setup from subsection \ref{s5.1} that allow for variation in the location parameter, and present the experimental results accordingly. In subsection \ref{s5.2.2}, we apply the CTE approach to the same distributions as in subsection \ref{s5.2.1}, and demonstrate that it allows for correct estimation of shape parameters. Additional experiments, in more simplified settings, highlighting the necessity of CTE are provided in Appendix F.

\subsubsection{Applying POT directly when the  location of conditional distributions exhibits substantial variability}\label{s5.2.1}

In order to ensure high variability of the location of conditional distributions $F_z(x)$, we modify step (c) of the sampling process in Subsection \ref{experimental_setting} as follows:

\begin{figure}[ht]
\centering
\includegraphics[scale=0.27]{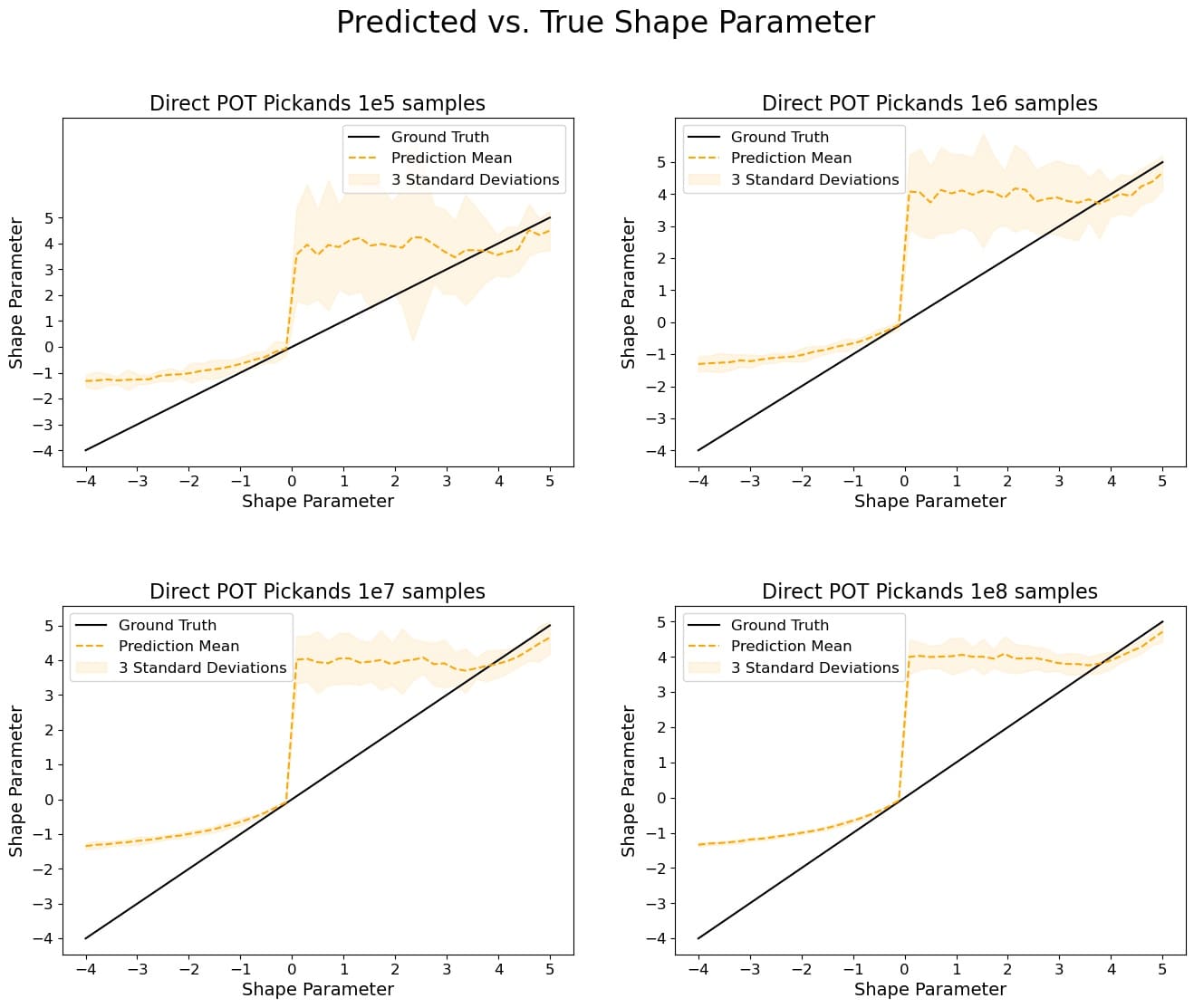}
\caption{The direct application of POT fails to retrieve the true shape of the marginal. We utilize the Pickands estimator as our estimator of choice.}
\label{lexp2_pot_pic}
\end{figure}  
\begin{enumerate}
  \item[(c*)] For the given  ${\xi_{z,j}}$ sample a point $x$ from a Generalized Pareto if  ${\xi_{z,j}}\leq 0$, otherwise sample from $F_z(x)=1-x^{-\frac{1}{\xi_{z,j}}}$. If ${\xi_{z,j}}>0$ translate $x$ by adding $\frac{1}{{\xi_{z,j}}^4}$. Add this sample to List J.
\end{enumerate}
This adaptation ensures that conditional distributions with lower shape parameters are situated at greater distances from the origin, thereby augmenting the probability that their tails will dominate over those that exhibit heavier tails.

The results (Figure \ref{lexp2_pot_pic} and \ref{lexp2_pot_dedh}) show that the estimators predict that the shape parameter of the tail is constantly $4$ as the tail of the marginal is determined by ${\xi_{z,j}}^{-4}$ instead of $1-x^{-\frac{1}{\xi_{z,j}}}$ which merely becomes noise around $\xi_{z,j}^{-4}$. This changes once $\xi_{\max}={\xi_{j}}$ becomes larger than $4$, in which case the tails of the conditional distribution are once again determined by $1-x^{-\frac{1}{\xi_{z,j}}}$.

\begin{figure}[ht]
\centering
\includegraphics[scale=0.27]{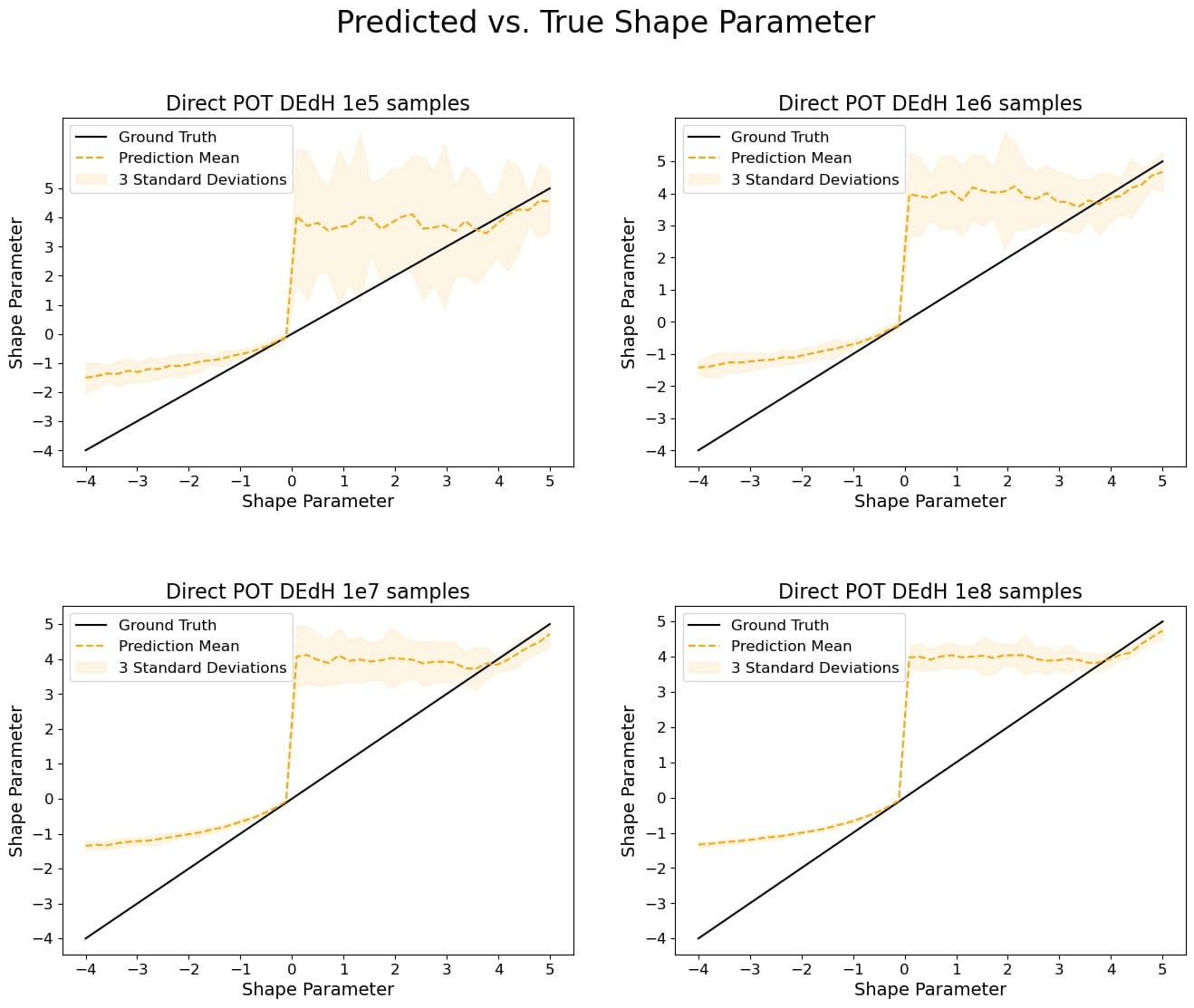}
\caption{The direct application of POT fails to retrieve the true shape of the marginal. We utilize the DEdH estimator as our estimator of choice.}
\label{lexp2_pot_dedh}
\end{figure}  

\subsubsection{Enhancing parameter estimation accuracy through the CTE approach}\label{s5.2.2}

We demonstrate that the CTE method can effectively recover the true shape of the tail of the marginal, even in cases where the conditional distributions exhibit highly varying locations, as was observed in the previous subsection. To ensure objectivity, we define the functions $f(z)$, ${\xi_z}$, and $F_z(x)$ in a consistent manner as before, thereby ensuring that all marginal distributions under consideration are equivalent to those studied in previous cases. As per the definition of the CTE, the sampling procedure is the following:
\begin{figure}[ht]
\centering
\includegraphics[scale=0.27]{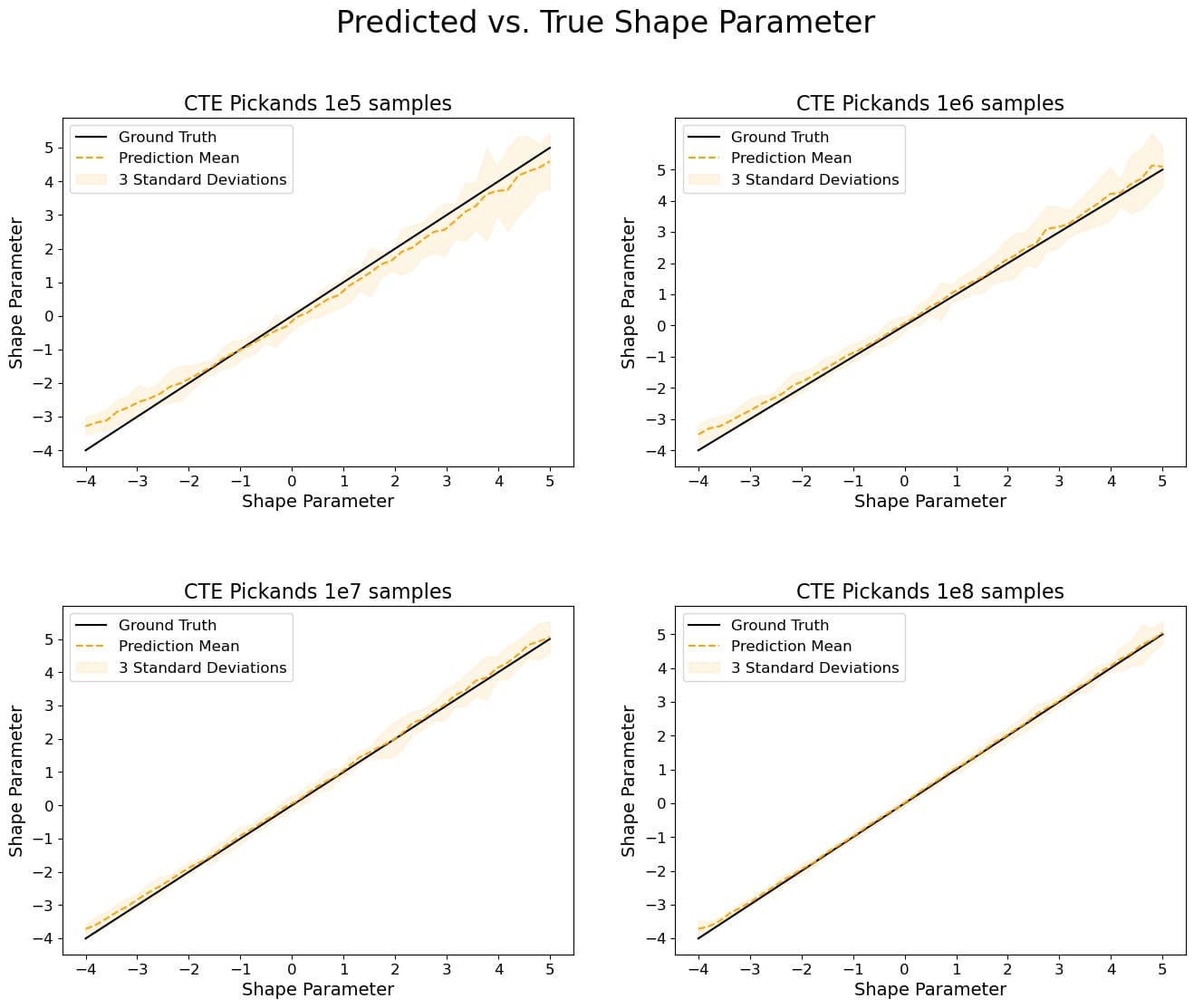}
\caption{Retrieving the true shape of the marginal is possible using CTE. We utilize the Pickands estimator as our estimator of choice.}
\label{lexp2_cte_pic}
\end{figure}  
\begin{enumerate}
  \item Sample K values $z$ from $f(z)$. For each $z$ repeat $p$ times the steps (a), (b), (c).
  \begin{enumerate}
  \item Calculate $\xi_{z,j}$ (given that $\xi_{\max}={\xi_{j}}$)
  \item For the given  $\xi_{z,j}$ sample N points from a Generalized Pareto if  $\xi_{z,j}\leq 0$, otherwise sample from $F_z(x)=1-x^{-\frac{1}{\xi_{z,j}}}$.
  \item Use the Pickands or DEdH estimator on these $N$ samples  to get an estimate ${\hat{\xi}^l}_{z,j}$ of the shape parameter ${\hat{\xi}}_{z,j}$ of $F_{z,j}(x)$. 
  \end{enumerate}
  \item As guided by the ideas laid in Appendix E, our final estimation of ${\xi_{z,j}}$ after $p$ repetitions of the process above is ${\hat{\xi}}_{z,j}=\frac{1}{p}\sum_{l=1}^{p} {\hat{\xi}^l}_{z,j}$. 
    \item We select the maximal ${\hat{\xi}}_{z,j}$ from the $K$ predicted values (corresponding to the $K$ sampled $z$). According to Theorem \ref{otheorem2} this estimated maximal ${\hat{\xi}}_{j}$ should be close to ${\xi_{j}}$.
\end{enumerate}
We set $p=10$ at all times. Furthermore, for the sake of fairness, we sample the same number of points from each marginal distribution as in the previous subsection, that is, we set $KN=M$. Since we set $K=50$, in order for $M$ to take values in $\{1e5,1e6,1e7,1e8\}$, $N$ needs to take values in $\{2e3, 2e4, 2e5, 2e6\}$. We execute the experiment 10 times, and to account for variability across the different runs, we compute the mean and standard deviation of the results. They are shown in Figure \ref{lexp2_cte_pic} and \ref{lexp2_cte_dedh}. Naturally, the more $K$ is increased the more likely we are to sample the $z$ corresponding to the conditional distribution with the maximal shape parameter. Hence, Theorem \ref{otheorem2} provides assurance that as the value of $K$ increases, our estimation progressively converges to the true shape parameter of the marginal distribution.

\begin{figure}[ht]
\centering
\includegraphics[scale=0.27]{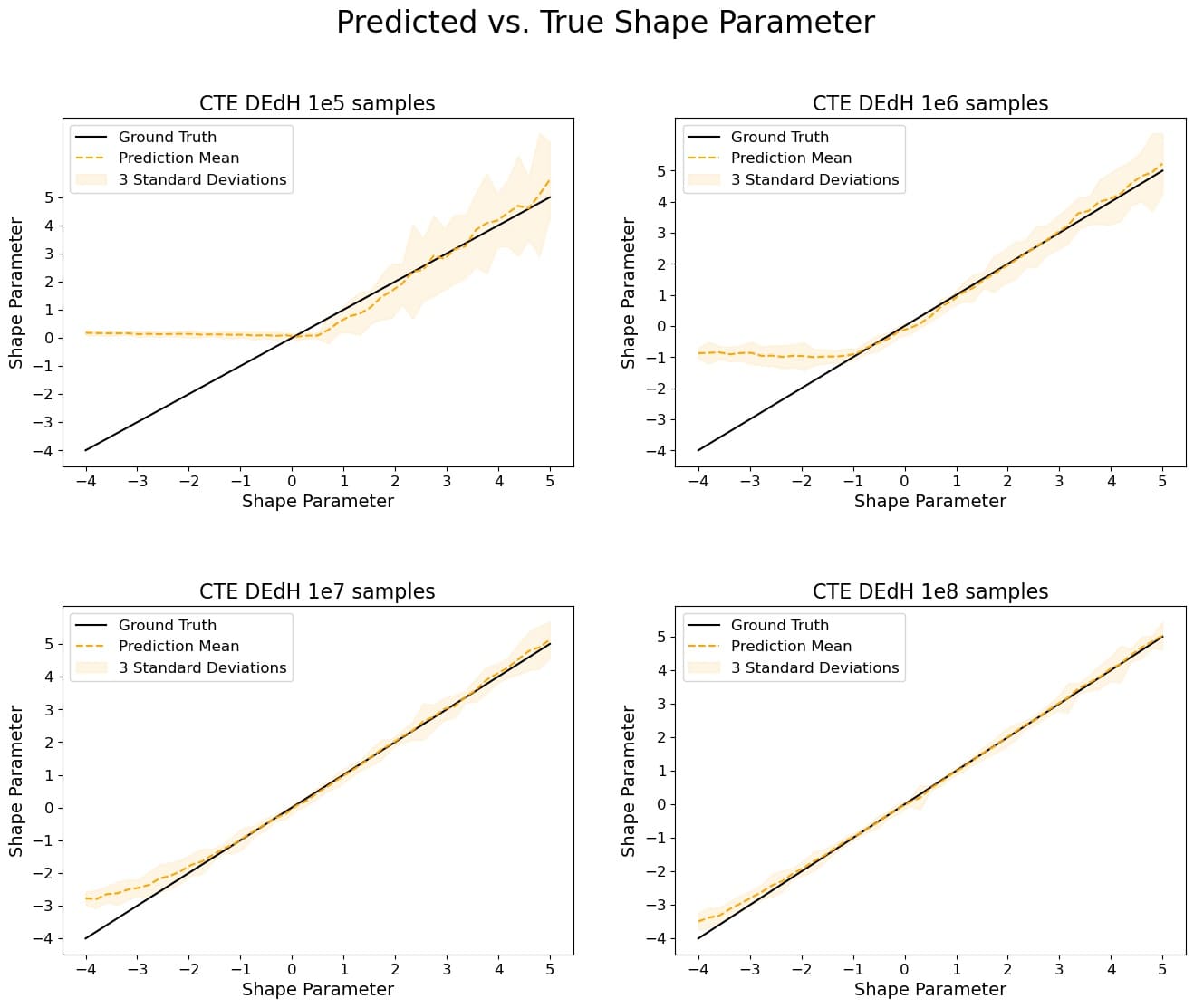}
\caption{Retrieving the true shape of the marginal is possible using CTE. We utilize the DEdH estimator as our estimator of choice.}
\label{lexp2_cte_dedh}
\end{figure}  

\subsection{Model performance inference improvements via cross tail estimation, relative to POT}\label{realdataexp}
In what follows, we show the results of two experiments, where we observe that cross tail estimation can improve the estimation of the shape of the tail in realistic settings. Furthermore, we observe that in these cases, the thickness of the tail is positively correlated with over-fitting, therefore inference regarding the performance of the model is improved when using CTE instead of POT.
\subsubsection{Gaussian Processes}
In this experiment, our data is composed of a one-dimensional time series taken from the UCR Time Series Anomaly Archive \footnote{\url{https://www.cs.ucr.edu/~eamonn/time_series_data_2018/UCR_TimeSeriesAnomalyDatasets2021.zip}} \cite{DBLP:journals/corr/abs-2009-13807}, which we reorganize in windows of size $2$, and use each window to fit a Gaussian process (GP) model in order to predict the next value in the series. Our complete dataset $D$ is composed of $n=1e4$ windows. 
\begin{figure}[ht]
\begin{tabular}{ll}
\includegraphics[scale=0.18]{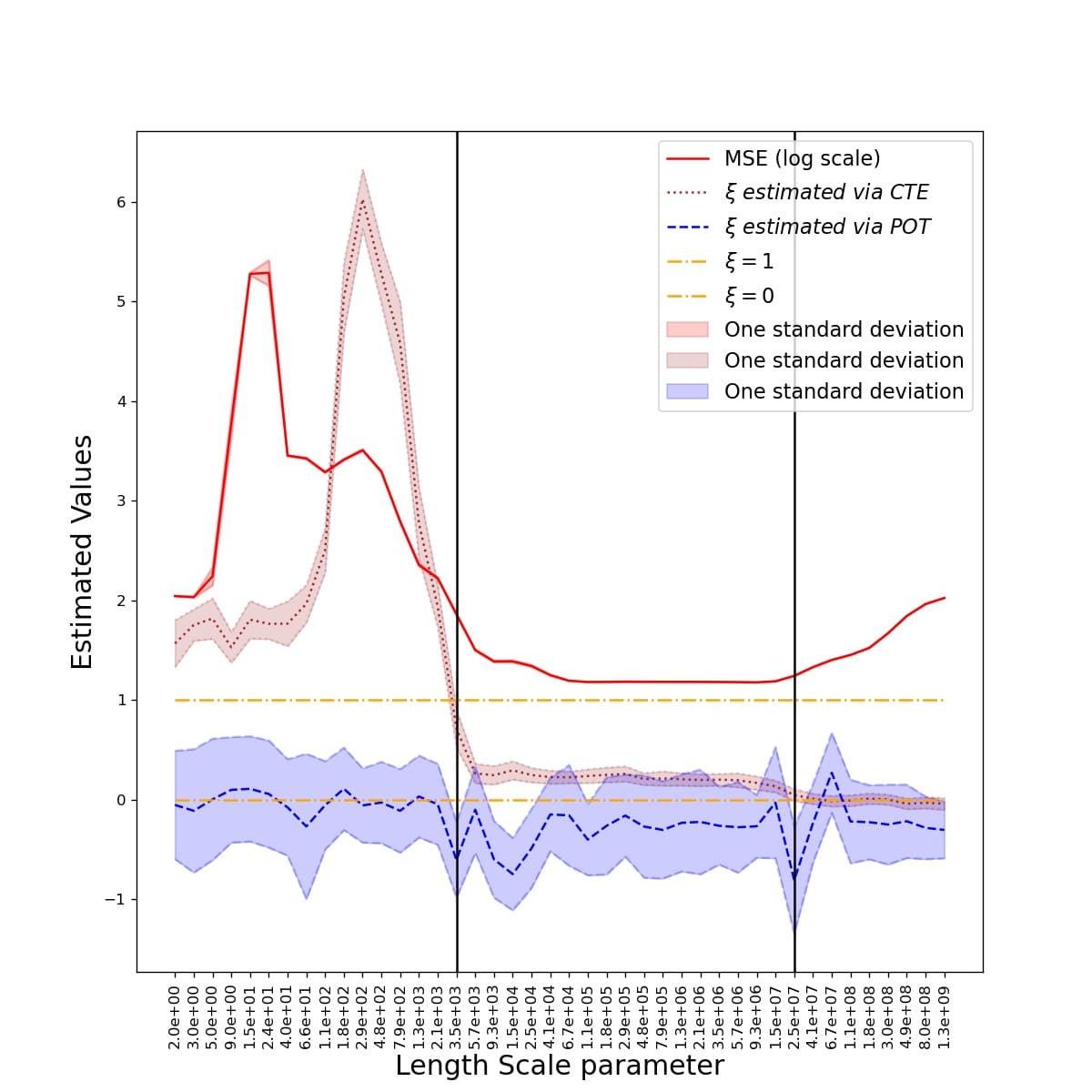}
&
\includegraphics[scale=0.18]{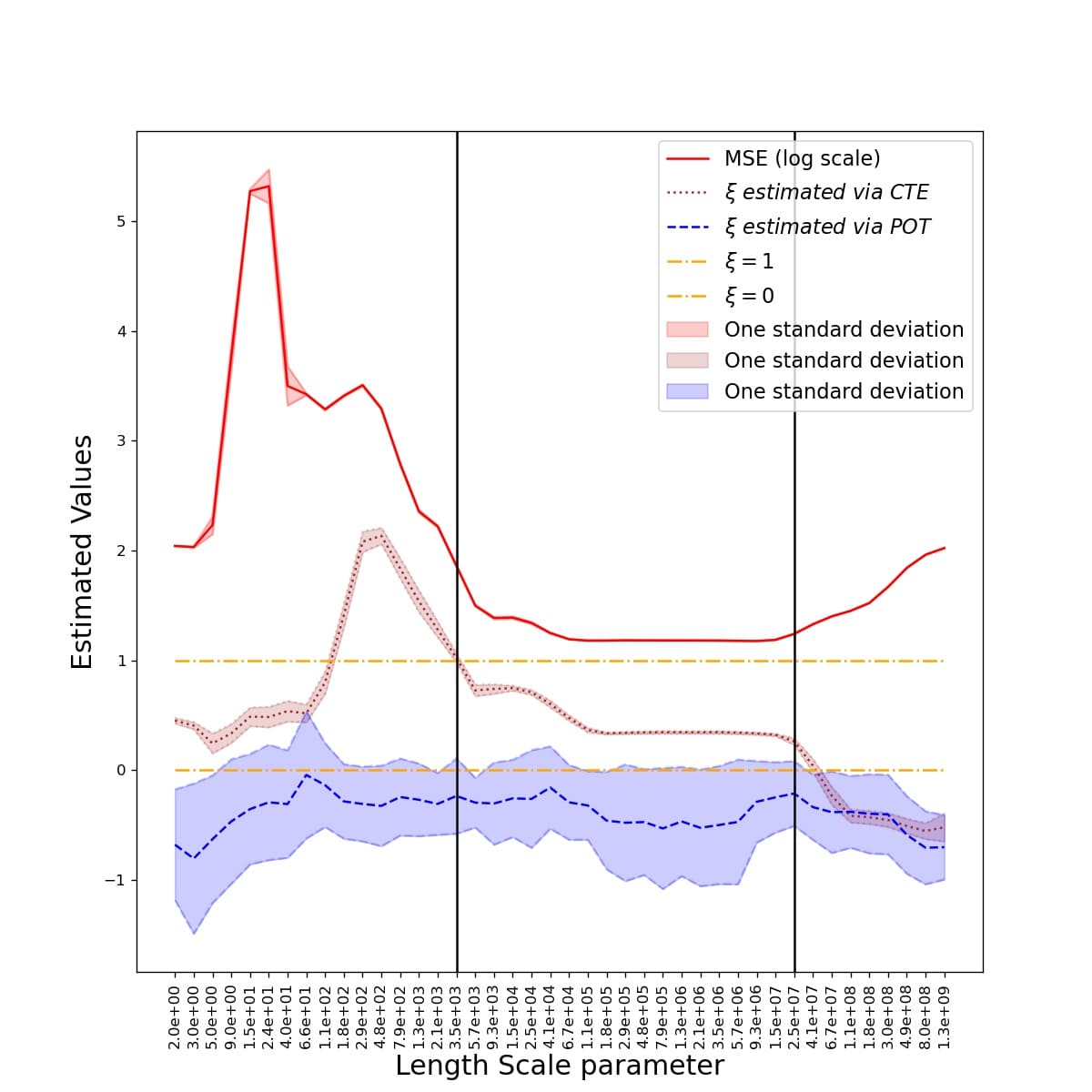}
\end{tabular}
\caption{ Experimental results in the case of testing Gaussian processes. Left: The Pickands estimator is used. Right: The DEdH estimator is used. In both cases we notice that CTE estimates larger shape parameters of the loss function distributions for models which overfit. This is not the case when POT is applied directly. The first black vertical line marks the first model with lower MSE than the model with the smallest length scale parameter (the point where the models stop overfitting). The second black vertical line marks the model in from which MSE starts growing again (the point when models begin underfitting). The MSE is presented in log scale and has been further linearly scaled to fit the plot.}
\label{fig5.3.1}
\end{figure}
On each run we randomly select $340$ points of $D$ for training (denote $D_{i}$), and then group the predictions of the model on the $1e4$ points of $D$ into an array which we denote by $\hat{Y}_i$. Then we split $\hat{Y}_i$ into five equally sized subsets $\hat{Y}_{i,j}$. We proceed to estimate the shape parameter of the tails of the prediction of the model, for given training set $D_{i}$. This is done by applying the Pickands/DEdH estimator to $\hat{Y}_{i,j}$, receiving $\hat{\xi}_{i,j}$ and then as per Appendix E, we get the estimate $\hat{\xi}_i=\frac{1}{5}\sum_{j=1}^5\hat{\xi}_{i,j} $ which corresponds to $\hat{Y}_i$. We repeat this process $1000$ times (for $1000$ choices of the training set $D_i$), and select as our estimation of the shape parameter of the tail of the distribution of our loss function, the maximum individual estimated parameter: $\hat{\xi_i}=\max\{\hat{\xi}_i|i\in [1000]\}$. On the other hand, we also calculate the MSE on the testing set $D\setminus D_{i}$ after the model has been trained on $D_{i}$. 
\newline
\newline
To check the difference of performance of the direct POT of tail shape estimation and cross tail estimation,  we also calculate the shape parameter of the overall distribution of prediction models, through the standard method, by applying Pickands/DEdH estimator on $Y = \bigcup\limits_{i=1}^{1000} \hat{Y}_i$. 
\newline
\newline
These experiments are repeated for length scale parameters given in the $x-$axis of Figure \ref{fig5.3.1} as well as in Appendix H. We repeat every experiment 200 times to account for variability across different runs, we compute the mean and standard deviation of the results. 
\newline
\newline
In Figure \ref{fig5.3.1} , we notice that when the CTE approach is used, the shape parameter is significantly larger for models which have a large MSE. In Appendix H, we illustrate that the MSE is large for small scale parameters due to overfitting (Figure \ref{gridgauss1}). Furthermore, the shape parameter only drops to (under) zero, when the model starts underfitting for length scale parameters bigger than $2.5e7$. In Appendix H (Figure \ref{gridgauss2}), it is shown that for such large values of the length scale parameter, the predictions become roughly constant. 
\newline
\newline
On the other hand, if POT is applied directly, then the estimated shape parameters are not significantly larger for models which overfit compared to those that do not. This is because conditioning on the training set, the predicted values on the test set vary significantly with regards to the their location. Hence, the tail that is estimated by the direct application of the POT approach is sometimes simply the one translated the furthest from the origin. Thus, if there is some negative correlation between the magnitude of the location and the size of the estimated shape parameter across different conditional, then we expect POT to underestimate the true shape parameter of the marginal. This is shown in Appendix H, for the model with the highest estimated shape parameter (290). The variability (sorted) of the estimated shape parameters of the 1000 conditionals for each length scale parameter is given in Figure \ref{variabilitygauss} of Appendix H, together with the corresponding $97th$ percentile (threshold) from each corresponding conditional distribution. We notice that indeed, quite often the difference between locations is large, and that the largest threshold often corresponds to conditional distributions with small, even negative shape parameters. 
\newline
\newline
The outcomes presented herein are robust with regards to the application or non-application of the method explicated in Appendix E in conjunction with the direct Peaks Over Threshold (POT) approach. Furthermore, the findings presented in Figure \ref{fig5.3.1} demonstrate near equivalence in relation to the magnitude of the selected threshold (in this study, we evaluated 99.7 and 99.997 percentiles).

\subsubsection{Polynomial Kernels}
This experiment is almost identical to the previous one, with the only differences being that the models we test now are polynomial kernels, and the set of possible candidate models in this case is defined by the degree of the polynomial kernel. We test polynomial kernels of degree from $1$ to $9$. 
As before, we repeat this experiment 200 times. The results are shown in Figure \ref{fig4.3.2}.

\begin{figure}[ht]
\begin{tabular}{ll}
\includegraphics[scale=0.18]{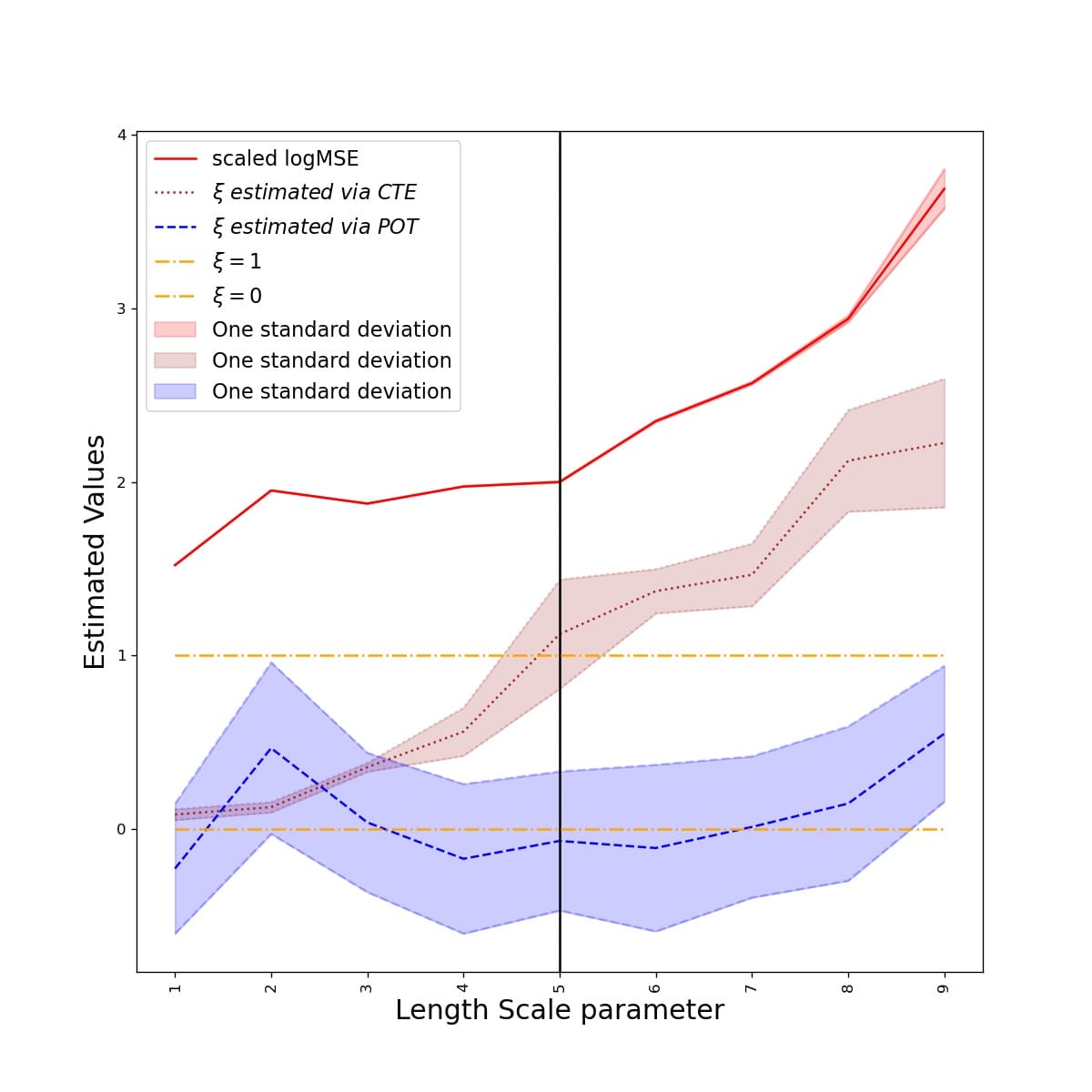}
&
\includegraphics[scale=0.18]{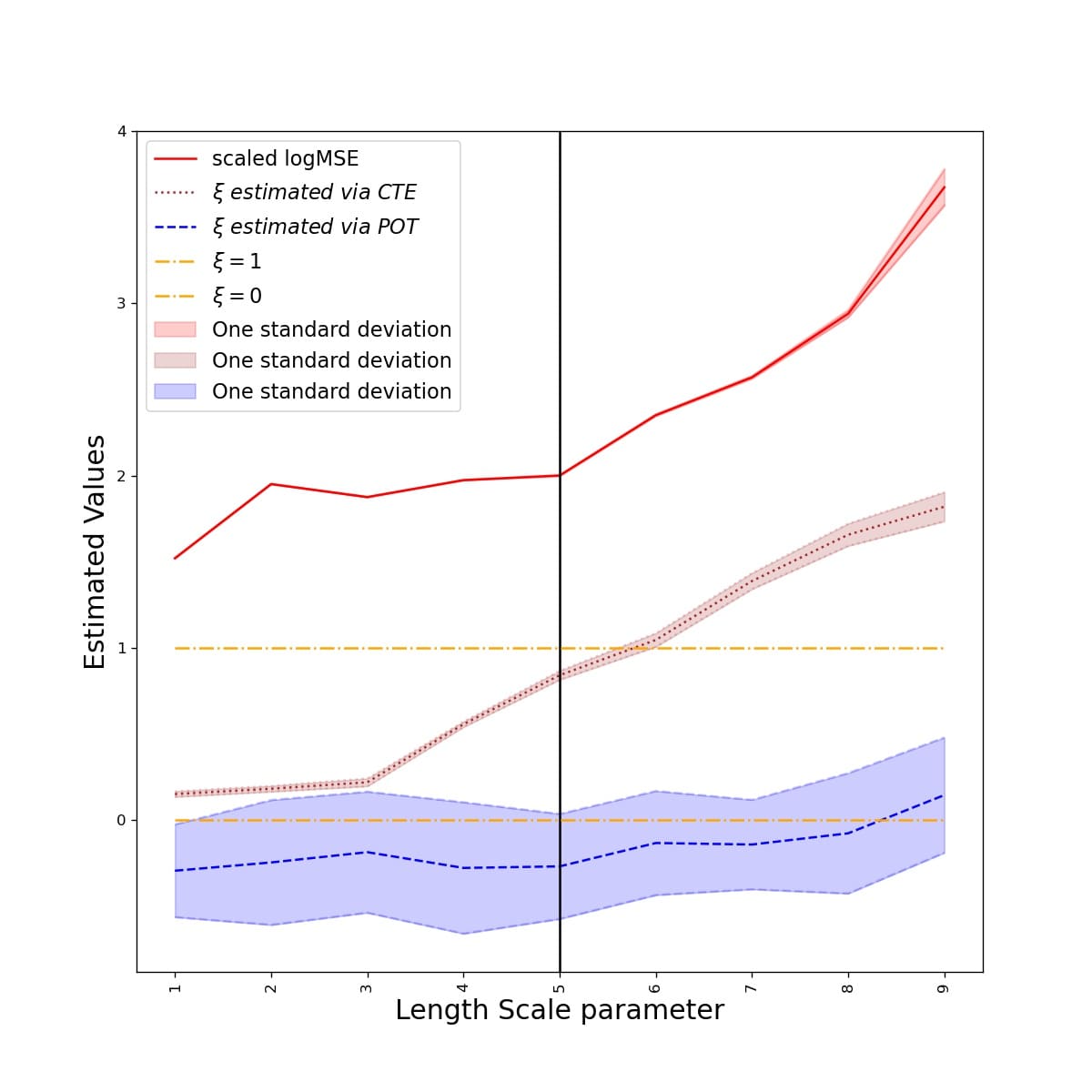}
\end{tabular}
\caption{ Experimental results in the case of testing polynomial kernels. Left: The Pickands estimator is used. Right: The DEdH estimator is used. In both cases we notice that CTE estimates larger shape parameters of the loss function distributions for models which overfit. This is not the case when POT is applied directly. The black vertical line marks the infection point of the MSE.  The MSE is presented in log scale and has been further linearly scaled to fit the plot.}
\label{fig4.3.2}
\end{figure}

\subsection{Computational Simplifications}
Another benefit to using cross tail estimation is the reduction of computational time, as for a given number $m$ of conditional distributions, with $n$ samples for each, instead of joining all testing samples together in an array of size $m*n$, we perform calculations in $m$ arrays of size $n$ in parallel. This becomes useful in practice during shape parameter estimation, as using Pickands estimators requires sorted samples, where best algorithms for sorting require $n \log (n)$ operations for a vector of size $n$. Hence our method which requires $n \log (n)$ operations is much faster in practice than the standard POT approach which requires $mn \log(mn)$, in a setting where $m$ and $n$ are of approximately of the same order. 

\section{Conclusion}

We study the problem of estimating the tail shape of loss function distributions, and explain the complications that arise in performing this task. We notice that such complications arise in general during the estimation of the tail shape of marginal distributions. In order to mitigate such shortcomings, we propose a new method of estimating the shape of the right tails of marginal distributions and  give  theoretical  guarantees  that  the  tail  of  the  marginal  distribution  coincides  with the thickest tail of the set of conditional distributions composing the marginal. We give experimental evidence that our method works in practice, and is necessary in applications with small sample sizes. Using the aforementioned method, we show experimentally that the tails of distribution functions in many cases can have non-exponential decay, as well as that it is possible that not
even their first moment exists. Furthermore, we discover an interesting phenomena regarding the relationship between the overfitting of a model, and the thickness of the tails of its prediction function distribution, in the experiments we conducted.

Potential additional applications of the method we develop include improving classic tail modelling, as well as  the threshold selection for model comparison in anomaly detection \cite{10.1145/3292500.3330672}. Furthermore, cross tail estimation could be used to estimate the existence of the moments of loss function distributions, and thus can be considered as a potential elimination criteria for models whose first moment does not exist.

\acks{This work has been supported by the French government, through the 3IA Côte d’Azur
Investments in the Future project managed by the National Research Agency (ANR) with
the reference number ANR-19-P3IA-0002. The authors are grateful to the OPAL infrastructure from Université Côte d'Azur for providing resources and support.}
\newpage
\appendix
\section*{Appendix A: Proofs}\label{AppendixA}
\label{app:theorem}

\subsection*{Proof of Proposition \ref{prop1}}

We notice that if $L(x)$ converges the statement is trivial. However, if it does not then:
\begin{equation}\label{c1_2}
\begin{split}
\lim_{x \rightarrow \infty}x^{- \epsilon} L(x)=\lim_{x \rightarrow \infty}\frac{L(x)}{x^{\epsilon}}= \lim_{x \rightarrow \infty}\frac{ e^{c(x)}e^{\int_{x_o}^x \frac{u(y)}{y} dy}}{x^{\epsilon}}=\lim_{x \rightarrow \infty}\frac{ e^{c(x)}e^{\int_{x_o}^x \frac{u(y)}{y} dy}}{ e^{\epsilon \log(x)}}=\\ =\lim_{x \rightarrow \infty}{ e^{c(x)}e^{\int_{x_o}^x \frac{u(y)}{y} dy-\epsilon\log(x)}}=\lim_{x \rightarrow \infty}{ e^{c(x)}e^{\log(x)(\frac{\int_{x_o}^x \frac{u(y)}{y} dy}{\log(x)}-\epsilon)}}.
\end{split}
\end{equation}
Using L'Hopital's rule we get:
\begin{equation}\label{c1_4}
 \lim_{x \rightarrow \infty}\frac{ \int_{x_o}^x \frac{u(y)}{y}}{\log(x)}= \lim_{x \rightarrow \infty} \frac{\frac{u(x)}{x}}{\frac{1}{x}}=\lim_{x \rightarrow \infty} u(x)=0,
\end{equation}
therefore 
\begin{equation}\label{c1_5}
\lim_{x \rightarrow \infty} e^{\log(x)(\frac{\int_{x_o}^x \frac{u(y)}{y} dy}{\log(x)}-\epsilon)}=0.
\end{equation}

\subsection*{Proof of Lemma \ref{lemma1}}

 From Theorem \ref{mda_rv}, we get that $$F_1 \in MDA(\xi_1) \iff \bar{F}_1(x)=x^{-\frac{1}{\xi_1}}L_1(x),$$ and $$F_2 \in MDA(\xi_2) \iff \bar{F}_2(x)=x^{-\frac{1}{\xi_2}}L_2(x),$$ where $L_1(x)$ and $L_2(x)$ are slowly varying functions.
\newline
Therefore 
\begin{equation}\label{l1_1}
\lim_{x \rightarrow \infty} \frac{\bar{F}_2(x)}{\bar{F}_1(x)}=\lim_{x \rightarrow \infty}x^{\frac{1}{\xi_1}-\frac{1}{\xi_2}} \frac{L_2(x)}{L_1(x)}=\lim_{x \rightarrow \infty}x^{\alpha} \frac{L_2(x)}{L_1(x)},
\end{equation}
since $$ \xi_1>\xi_2 \implies -\frac{1}{\xi_1}>-\frac{1}{\xi_2} \implies \alpha:=\frac{1}{\xi_1}-\frac{1}{\xi_2}<0.$$
On the other hand $L(x):=\frac{L_2(x)}{L_1(x)}$ is defined in a neighborhood of infinity as $L_1(x)\neq0$, and is also a slowly varying function as $$\lim_{x \rightarrow \infty} \frac{L(ax)}{L(x)}=\lim_{x \rightarrow \infty}\frac{\frac{L_2(ax)}{L_1(ax)}}{\frac{L_2(x)}{L_1(x)}}=\lim_{x \rightarrow \infty}\frac{\frac{L_2(ax)}{L_2(x)}}{\frac{L_1(ax)}{L_1(x)}}=1,$$ and since the quotient of positive measurable functions, is positive and measurable. Therefore, using Corollary 1, Equation (\ref{l1_1}) becomes
\begin{equation}\label{l1_2}
\lim_{x \rightarrow \infty} \frac{\bar{F}_2(x)}{\bar{F}_1(x)}=\lim_{x \rightarrow \infty}x^{\alpha} \frac{L_2(x)}{L_1(x)}=\lim_{x \rightarrow \infty}x^{\alpha} L(x)=0.
\end{equation}

\subsection*{Proof of Lemma \ref{lemma2}}
\begin{enumerate}  
\item  If $\xi_1>0$ and $\xi_2=0$ then 
\begin{equation}\label{l2_1}
\begin{split}
\lim_{x \rightarrow \infty} \frac{\bar{F}_2(x)}{\bar{F}_1(x)}=\lim_{x \rightarrow \infty} \frac{c(x)e^{-\int_w^x \frac{g(t)}{a(t)}dt}}{x^{-\frac{1}{\xi}}L(x)}
=\lim_{x \rightarrow \infty} \frac{c(x)e^{-\log(x)(\frac{\int_w^x \frac{g(t)}{a(t)}dt}{\log(x)}-\frac{1}{\xi})}}{L(x)},
\end{split}
\end{equation}
using L'Hopital's rule:
\begin{equation}\label{l2_2}
\begin{split}
\lim_{x \rightarrow \infty} \frac{\int_w^x \frac{g(t)}{a(t)}dt}{\log(x)}=\lim_{x \rightarrow \infty} \frac{ \frac{g(x)}{a(x)}}{\frac{1}{x}}=\lim_{x \rightarrow \infty}  \frac{x}{a(x)},
\end{split}
\end{equation}
we distinguish two cases:

if $\lim_{x \rightarrow \infty}  a(x)\neq \infty$  then $\lim_{x \rightarrow \infty}  \frac{x}{a(x)} = \infty$,

while if $\lim_{x \rightarrow \infty}  a(x)= \infty$  then using L'Hopital's rule again, we obtain
\begin{equation}\label{l2_additional1}
\lim_{x \rightarrow \infty}  \frac{x}{a(x)} = \lim_{x \rightarrow \infty}  \frac{1}{a'(x)} = \infty.
\end{equation}
Thus, in both cases
\begin{equation}
    =\lim_{x \rightarrow \infty} \frac{c(x)e^{-\log(x)(\frac{\int_w^x \frac{g(t)}{a(t)}dt}{\log(x)}-\frac{1}{\xi})}}{L(x)}
    =\lim_{x \rightarrow \infty} \frac{c(x)x^{-(\frac{\int_w^x \frac{g(t)}{a(t)}dt}{\log(x)}-\frac{1}{\xi})}}{L(x)}=0.
\end{equation}
\end{enumerate}
Statements 2. 3. and 4. are trivial.
\subsection*{Proof of Lemma \ref{lemma3}}
Since $L(x)$ is positive and measurable (linear combination of finite measurable functions), the only part left to prove is that $$\lim_{x\rightarrow\infty}\frac{L(ax)}{L(x)}=1, \forall a>0.$$
First we prove that $$\lim_{x\rightarrow\infty}\frac{L_1(ax)+L_2(ax)}{L_1(x)+L_2(x)}=1, \forall a>0.$$
Indeed, for each $\epsilon>0$, there exist $x_1, x_2$ such that for $x>x_1$ we have $|\frac{L_1(ax)}{L_1(x)}-1|<\epsilon$ and  for $x>x_2$ we have $|\frac{L_2(ax)}{L_2(x)}-1|<\epsilon$. Hence for $x_0=\max\{x_1,x_2\}$, $x>x_0$ implies
$|{L_1(ax)}-{L_1(x)}|<{L_1(x)}\epsilon$ and $|{L_2(ax)}-{L_2(x)}|<{L_2(x)}\epsilon$ therefore $|L_1(ax)+L_2(ax)-(L_1(x)+L_2(x))|=|L_1(ax)-L_1(x)+L_2(ax)-L_2(x)|\leq|L_1(ax)-L_1(x)|+|L_2(ax)-L_2(x)|<(L_1(x)+L_2(x))\epsilon$ hence $|\frac{L_1(ax)+L_2(ax)}{L_1(x)+L_2(x)}-1|<\epsilon$.
\newline
Now, we notice that for every $a_i>0$, we get $\lim_{x\rightarrow\infty}\frac{a_i L_i(ax)}{a_i L_i(x)}=1$, and ${a_i L_i(x)}$ is positive as well as measurable. This implies that $a_1L_1$ and $a_2L_2$ are slowly varying functions, and therefore based of the previous result we get $$\lim_{x\rightarrow\infty}\frac{a_1 L_1(ax)+a_2 L_2(ax)}{a_1 L_1(x)+a_2 L_2(x)}=1, \forall a>0.$$ Using induction finishes the proof of the Lemma.

\subsection*{Proof of Theorem \ref{otheorem1}}
Since if $\xi_{{\boldsymbol{z}}_i}<0$ then $\exists x_0>0$, such that $\forall x>x_0$ we have $F_{{\boldsymbol{Z}}_i}(x)=0$, this means that the tail of the distribution is not affected by $F_{{\boldsymbol{Z}}_i}(x)$. In fact if $\xi_{\max}<0$ then $F$ will have finite support hence $\xi_F\leq0$. Furthermore if $\xi_{\max}=0$ from Lemma \ref{lemma2} we get that $\xi_F\leq0$. Therefore for the case $\xi_{\max}>0$ we only consider the setting  where $\xi_i \geq 0$.
\begin{equation}\label{t1_2}
\bar{F}_u(w)=\frac{1-F(u+w)}{1-F(u)}=\frac{\sum\limits_i^n p_i (1-F_{{\boldsymbol{z}}_i}(u+w))}{\sum\limits_i^n p_i (1-F_{{\boldsymbol{z}}_i}(u))}=\sum\limits_i^n  \frac{ \bar{F}_{{\boldsymbol{z}}_i}(u+w)}{\sum\limits_j^n \frac{p_j}{p_i} \bar{F}_{{\boldsymbol{z}}_j}(u)}
\end{equation}
\begin{equation}\label{t1_3}
=\sum\limits_i^n  \frac{ \bar{F}_{{\boldsymbol{z}}_i}(u+w)}{\bar{F}_{{\boldsymbol{z}}_i}(u)}\frac{ \bar{F}_{{\boldsymbol{z}}_i}(u)}{\sum\limits_j^n \frac{p_j}{p_i} \bar{F}_{{\boldsymbol{z}}_j}(u)}=\sum\limits_i^n \frac{ \bar{F}_{{\boldsymbol{z}}_i}(u+w)}{\bar{F}_{{\boldsymbol{z}}_i}(u)} \frac{ 1}{\sum\limits_j^n \frac{p_j}{p_i} \frac{ \bar{F}_{{\boldsymbol{z}}_j}(u)}{\bar{F}_{{\boldsymbol{z}}_i}(u)}}.
\end{equation}
We denote with $i(\max)$ the index corresponding to $\xi_{\max}$ and finish our proof using Pickand's theorem:
\begin{equation}\label{t1_4}
\lim_{u \rightarrow \infty}\sup_{w\in[0,\infty]}|\bar{F}_u(y)-\bar{G}_{\xi_{\max},g(u)}|=\lim_{u \rightarrow \infty}\sup_{w\in[0,\infty]}|\sum\limits_i^n \frac{ \bar{F}_{{\boldsymbol{z}}_i}(u+w)}{\bar{F}_{{\boldsymbol{z}}_i}(u)} \frac{ 1}{\sum\limits_j^n \frac{p_j}{p_i} \frac{ \bar{F}_{{\boldsymbol{z}}_j}(u)}{\bar{F}_{{\boldsymbol{z}}_i}(u)}}-\bar{G}_{\xi_{\max},g(u)}|
\end{equation}
\begin{equation}\label{t1_5}
=\lim_{u \rightarrow \infty}\sup_{w\in[0,\infty]}|\sum\limits_i^n \frac{ \bar{F}_{{\boldsymbol{z}}_i}(u+w)}{\bar{F}_{{\boldsymbol{z}}_i}(u)} \frac{ 1}{1+\sum\limits^n_{j\neq i} \frac{p_j}{p_i} \frac{ \bar{F}_{{\boldsymbol{z}}_j}(u)}{\bar{F}_{{\boldsymbol{z}}_i}(u)}}-\bar{G}_{\xi_{\max},g(u)}|
\end{equation}
\begin{equation}\label{t1_6}
\begin{split}
& \leq\lim_{u \rightarrow \infty}\sup_{w\in[0,\infty]}| \frac{ \bar{F}_{{\boldsymbol{z}}_{i(\max)}}(u+w)}{\bar{F}_{{\boldsymbol{z}}_{i(\max)}}(u)} \frac{ 1}{1+\sum\limits^n_{j\neq {i(\max)}} \frac{p_j}{p_{i(\max)}} \frac{ \bar{F}_{{\boldsymbol{z}}_j}(u)}{\bar{F}_{{\boldsymbol{z}}_{i(\max)}}(u)}}-\bar{G}_{\xi_{\max},g(u)}|\\&
+\lim_{u \rightarrow \infty}\sup_{w\in[0,\infty]}|\sum\limits_{i\neq i(\max)}^n \frac{ \bar{F}_{{\boldsymbol{z}}_i}(u+w)}{\bar{F}_{{\boldsymbol{z}}_i}(u)} \frac{ 1}{1+\sum\limits^n_{j\neq i} \frac{p_j}{p_i} \frac{ \bar{F}_{{\boldsymbol{z}}_j}(u)}{\bar{F}_{{\boldsymbol{z}}_i}(u)}}|
\end{split}
\end{equation}
\begin{equation}\label{t1_7}
\begin{split}
&\leq\lim_{u \rightarrow \infty}\sup_{w\in[0,\infty]}| \frac{ \bar{F}_{{\boldsymbol{z}}_{i(\max)}}(u+w)}{\bar{F}_{{\boldsymbol{z}}_{i(\max)}}(u)} -\bar{G}_{\xi_{\max},g(u)}| \\&
+\lim_{u \rightarrow \infty}\sup_{w\in[0,\infty]}|\frac{ 1}{1+\sum\limits^n_{j\neq {i(\max)}} \frac{p_j}{p_{i(\max)}} \frac{ \bar{F}_{{\boldsymbol{z}}_j}(u)}{\bar{F}_{{\boldsymbol{z}}_{i(\max)}}(u)}}-1| |\frac{ \bar{F}_{{\boldsymbol{z}}_{i(\max)}}(u+w)}{\bar{F}_{{\boldsymbol{z}}_{i(\max)}}(u)}|\\&
+\lim_{u \rightarrow \infty}\sup_{w\in[0,\infty]}\sum\limits_{i\neq i(\max)}^n |\frac{ \bar{F}_{{\boldsymbol{z}}_i}(u+w)}{\bar{F}_{{\boldsymbol{z}}_i}(u)}| |\frac{ 1}{1+\sum\limits^n_{j\neq i} \frac{p_j}{p_i} \frac{ \bar{F}_{{\boldsymbol{z}}_j}(u)}{\bar{F}_{{\boldsymbol{z}}_i}(u)}}|
\end{split}
\end{equation}
\begin{equation}\label{t1_8}
\begin{split}
&\leq\lim_{u \rightarrow \infty}\sup_{w\in[0,\infty]}| \frac{ \bar{F}_{{\boldsymbol{z}}_{i(\max)}}(u+w)}{\bar{F}_{{\boldsymbol{z}}_{i(\max)}}(u)} -\bar{G}_{\xi_{\max},g(u)}|\\&
+\lim_{u \rightarrow \infty}|\frac{ 1}{1+\sum\limits^n_{j\neq {i(\max)}} \frac{p_j}{p_{i(\max)}} \frac{ \bar{F}_{{\boldsymbol{z}}_j}(u)}{\bar{F}_{{\boldsymbol{z}}_{i(\max)}}(u)}}-1|\\&
+\lim_{u \rightarrow \infty}\sum\limits_{i\neq i(\max)}^n |\frac{ 1}{1+\sum\limits^n_{j\neq i} \frac{p_j}{p_i} \frac{ \bar{F}_{{\boldsymbol{z}}_j}(u)}{\bar{F}_{{\boldsymbol{z}}_i}(u)}}|.
\end{split}
\end{equation}
The first expression,
\begin{equation}\label{t1_9}
\begin{split}
\lim_{u \rightarrow \infty}\sup_{w\in[0,\infty]}| \frac{ \bar{F}_{{\boldsymbol{z}}_{i(\max)}}(u+w)}{\bar{F}_{{\boldsymbol{z}}_{i(\max)}}(u)} -\bar{G}_{\xi_{\max},g(u)}|
\end{split}
\end{equation}
goes to zero due to Pickands Theorem while the expression,
\begin{equation}\label{t1_10}
\begin{split}
\lim_{u \rightarrow \infty}|\frac{ 1}{1+\sum\limits^n_{j\neq {i(\max)}} \frac{p_j}{p_{i(\max)}} \frac{ \bar{F}_{{\boldsymbol{z}}_j}(u)}{\bar{F}_{{\boldsymbol{z}}_{i(\max)}}(u)}}-1|
\end{split}
\end{equation}
converges to $0$ as well because from Lemma \ref{lemma1} we have  $\lim_{u\rightarrow\infty}\frac{ \bar{F}_{{\boldsymbol{z}}_j}(u)}{\bar{F}_{{\boldsymbol{z}}_{i(\max)}}(u)}=0$ for every $j$. Finally the last expression,
\begin{equation}\label{t1_11}
\begin{split}
\lim_{u \rightarrow \infty}\sum\limits_{i\neq i(\max)}^n |\frac{ 1}{1+\sum\limits^n_{j\neq i} \frac{p_j}{p_i} \frac{ \bar{F}_{{\boldsymbol{z}}_j}(u)}{\bar{F}_{{\boldsymbol{z}}_i}(u)}}|
\end{split}
\end{equation} equals 0 since in each sum $\sum\limits^n_{j\neq i} \frac{p_j}{p_i} \frac{ \bar{F}_{{\boldsymbol{z}}_j}(u)}{\bar{F}_{{\boldsymbol{z}}_i}(u)}$, there exists an index $j$ such that $\bar{F}_{{\boldsymbol{z}}_j}(u)=\bar{F}_{{\boldsymbol{z}}_{i(\max)}}(u)$, implying that $\sum\limits^n_{j\neq i} \frac{p_j}{p_i} \frac{ \bar{F}_{{\boldsymbol{z}}_j}(u)}{\bar{F}_{{\boldsymbol{z}}_i}(u)} \rightarrow \infty$.
\newline
In the derivation above we assumed that the $F_{{\boldsymbol{z}}_{i(\max)}}$ which corresponds to $\xi_{\max}$ is unique. In the case that this is not true we notice that for $F_1$ and $F_2$ which share the same corresponding parameter $\xi>0$ we have 
\begin{equation}\label{t1_12}
\begin{split}
p_1 F_1(x)+p_2 F_2(x)=x^{-\frac{1}{\xi}}(p_1 L_1(x)+p_2 L_2(x))=x^{-\frac{1}{\xi}}L(x),
\end{split}
\end{equation}
and since $L(x)>0$, from Lemma \ref{lemma3} we have that $L(x)$ is slowly varying, therefore $p_1 F_1(x)+p_2 F_2(x) \in MDA(\xi)$.

\subsection*{Proof of Proposition \ref{propUniSubPol}}
First, we fix $\delta>0$. We can find a $x(\gamma,\delta)>0$, such that for  $x>x(\gamma,\delta)$, we can bound $x^{-\delta}L_z(x)<\gamma$ for all $z \in A$ simultaneously. This implies that $f_Z(z) x^{-\delta} L_z(x)$ is bounded by $f_z(z)\gamma$.
Since $\int_z f_z(z)\gamma dz=\gamma<\infty$, by dominated convergence we get 
\begin{equation}\label{propUniSubPol_1}
\begin{split} 
\lim_{x\rightarrow\infty} x^{-\delta}\int_A f_Z(z) L_z(x) dz= \lim_{x\rightarrow\infty} \int_A f_Z(z) x^{-\delta} L_z(x) dz=
 \int_A \lim_{x\rightarrow\infty} f_Z(z) x^{-\delta} L_z(x) dz=0.
\end{split}
\end{equation}

\subsection*{Proof of Theorem \ref{otheorem2a}}

We will first assume that $\xi_F>0$. \newline
Since $\bar{F}(x)=x^{-\frac{1}{\xi_F}}L_F(x)$, for every $\epsilon>0$:
\begin{equation}\label{l3_1}
\begin{split}
\frac{\bar{F}(x)}{x^{-\frac{1}{\xi_{lo}-\epsilon}}}=\frac{x^{-\frac{1}{\xi_F}}L_F(x)}{x^{-\frac{1}{\xi_{lo}-\epsilon}}}=\frac{\int_A f_{\boldsymbol{Z}}({\boldsymbol{z}}) x^{-\frac{1}{\xi_{\boldsymbol{z}}}}L_{\boldsymbol{z}}(x)d{\boldsymbol{z}}}{x^{-\frac{1}{\xi_{lo}-\epsilon}}}=\\
\int_A f_{\boldsymbol{Z}}({\boldsymbol{z}})x^{-\frac{1}{\xi_{\boldsymbol{z}}}+\frac{1}{\xi_{lo}-\epsilon}}{L_{\boldsymbol{z}}(x)}d{\boldsymbol{z}}=\int_A f_{\boldsymbol{Z}}({\boldsymbol{z}})x^{\alpha({\boldsymbol{z}})}{L_{\boldsymbol{z}}(x)}d{\boldsymbol{z}}.
\end{split}
\end{equation}
We notice that $\xi_{\boldsymbol{z}}\geq \xi_{lo}>\xi_{lo}-\epsilon \implies -\frac{1}{\xi_{\boldsymbol{z}}}\geq -\frac{1}{\xi_{lo}}>-\frac{1}{\xi_{lo}-\epsilon}$ hence $\alpha({\boldsymbol{z}})=-\frac{1}{\xi_{{\boldsymbol{z}}}}+\frac{1}{\xi_{lo}-\epsilon}>0$.
Considering that
\begin{equation}\label{l3_2}
\begin{split}
\lim_{x\rightarrow\infty}\frac{\bar{F}(x)}{x^{-\frac{1}{\xi_{lo}-\epsilon}}}=\lim_{x\rightarrow\infty} \int_A f_{\boldsymbol{Z}}({\boldsymbol{z}})x^{\alpha({\boldsymbol{z}})}{L_{\boldsymbol{z}}(x)}dz,
\end{split}
\end{equation}
by using Fatou's lemma:
\begin{equation}\label{l3_3}
\begin{split}
\lim_{x\rightarrow\infty} \int_A f_{\boldsymbol{Z}}({\boldsymbol{z}})x^{\alpha({\boldsymbol{z}})}{L_{\boldsymbol{z}}(x)}d{\boldsymbol{z}}\geq \int_A\lim_{x\rightarrow\infty} f_{\boldsymbol{Z}}({\boldsymbol{z}})x^{\alpha({\boldsymbol{z}})}{L_{\boldsymbol{z}}(x)}d{\boldsymbol{z}}=\infty,
\end{split}
\end{equation}
we get
\begin{equation}\label{l3_4}
\lim_{x\rightarrow\infty}\frac{x^{-\frac{1}{\xi_{lo}-\epsilon}}}{\bar{F}(x)}=0,
\end{equation}
implying 
\begin{equation}\label{l3_5}
\lim_{x\rightarrow\infty}\frac{x^{-\frac{1}{\xi_{lo}-\epsilon}}}{x^{-\frac{1}{\xi_F}}L_F(x)}=\lim_{x\rightarrow\infty}\frac{x^{-\frac{1}{\xi_{lo}-\epsilon}+\frac{1}{\xi_F}}}{L_F(x)}=0,
\end{equation}
therefore 
\begin{equation}\label{l3_6}
\xi_{lo}-\epsilon<\xi_F, \forall \epsilon>0 \text{ thus } \xi_{lo}\leq\xi_F.
\end{equation}
Now we turn to prove that $\xi_F\leq\xi_{up}$. As before,
\begin{equation}\label{l3_7}
\begin{split}
\frac{\bar{F}(x)}{x^{-\frac{1}{\xi_{up}+\epsilon}}}=\frac{x^{-\frac{1}{\xi_F}}L_F(x)}{x^{-\frac{1}{\xi_{up}+\epsilon}}}=\frac{\int_A f_{\boldsymbol{Z}}({\boldsymbol{z}}) x^{-\frac{1}{\xi_{\boldsymbol{z}}}}L_{\boldsymbol{z}}(x)d{\boldsymbol{z}}}{x^{-\frac{1}{\xi_{up}+\epsilon}}}=\\
\int_A f_{\boldsymbol{Z}}({\boldsymbol{z}})x^{-\frac{1}{\xi_{\boldsymbol{z}}}+\frac{1}{\xi_{up}+\epsilon}}{L_{\boldsymbol{z}}(x)}d{\boldsymbol{z}}=\int_A f_{\boldsymbol{Z}}({\boldsymbol{z}})x^{\beta({\boldsymbol{z}})}{L_{\boldsymbol{z}}(x)}d{\boldsymbol{z}}.
\end{split}
\end{equation}
We notice that $\xi_{\boldsymbol{z}}\leq \xi_{up}<\xi_{up}+\epsilon \implies -\frac{1}{\xi_{\boldsymbol{z}}}\leq -\frac{1}{\xi_{up}}<-\frac{1}{\xi_{up}+\epsilon}$ hence $\beta(\boldsymbol{z})=-\frac{1}{\xi_{\boldsymbol{z}}}+\frac{1}{\xi_{up}+\epsilon}<-\delta<0$. 
This last inequality, combined with the fact that the family $\{L_{\boldsymbol{z}}(x)|x\in \mathbb{R}\}$ is $\gamma$-uniformly sub-polynomial, implies that 
\begin{equation}\label{l3_7*}
\begin{split} 
f_{\boldsymbol{Z}}({\boldsymbol{z}})x^{\beta({\boldsymbol{z}})}{L_{\boldsymbol{z}}(x)}\leq f_{\boldsymbol{Z}}({\boldsymbol{z}})x^{-\delta}{L_{\boldsymbol{z}}(x)}\leq f_{\boldsymbol{Z}}({\boldsymbol{z}})\gamma,
\end{split}
\end{equation}
for some $\gamma>0$.
Since $\int_{\boldsymbol{z}} f_{\boldsymbol{Z}}({\boldsymbol{z}})\gamma d{\boldsymbol{z}}=\gamma<\infty$, by dominated convergence
\begin{equation}\label{l3_8}
\begin{split} 
\lim_{x\rightarrow\infty}  \frac{\bar{F}(x)}{x^{-\frac{1}{\xi_{up}+\epsilon}}}=\lim_{x\rightarrow\infty} \int_A f_{\boldsymbol{Z}}({\boldsymbol{z}})x^{\beta({\boldsymbol{z}})}{L_{\boldsymbol{z}}(x)}d{\boldsymbol{z}}
\end{split}
\end{equation}

\begin{equation}\label{l3_9}
\begin{split}
\lim_{x\rightarrow\infty} \int_A f_{\boldsymbol{Z}}({\boldsymbol{z}})x^{\beta({\boldsymbol{z}})}{L_{\boldsymbol{z}}(x)}d{\boldsymbol{z}}= \int_A\lim_{x\rightarrow\infty} f_{\boldsymbol{z}}({\boldsymbol{z}})x^{\beta({\boldsymbol{z}})}{L_{\boldsymbol{z}}(x)}d{\boldsymbol{z}}=0,
\end{split}
\end{equation} meaning
\begin{equation}\label{l3_10}
\lim_{x\rightarrow\infty}\frac{\bar{F}(x)}{x^{-\frac{1}{\xi_{up}+\epsilon}}}=0,
\end{equation}
which implies
\begin{equation}\label{l3_11}
\lim_{x\rightarrow\infty}\frac{x^{-\frac{1}{\xi_F}}L_F(x)}{x^{-\frac{1}{\xi_{up}+\epsilon}}}=\lim_{x\rightarrow\infty}x^{\frac{1}{\xi_{up}+\epsilon}-\frac{1}{\xi_F}}{L_F(x)}=0,
\end{equation}
therefore we get
\begin{equation}\label{l3_12}
\xi_{up}+\epsilon>\xi_F, \forall \epsilon>0 \text{ hence } \xi_F\leq\xi_{up}.
\end{equation}
Now we prove that indeed $\xi_F>0$. It is simple to show that $\xi_F$ cannot be negative. Indeed, if $\xi_F$ is negative, it means that $F$ has finite support which is not possible as for each fixed $x$, we have $F_{\boldsymbol{z}}(x)>0,\forall {\boldsymbol{z}}\in A$, therefore $\forall x \in \mathbb{R}, F(x)>0$.
\newline
Proving that $\xi_F\neq0$ is slightly less trivial. For every distribution $G_0\in MDA(0)$ and for $\epsilon<\xi_{lo}$
\begin{equation}\label{l3_13}
\begin{split}
\frac{\bar{F}(x)}{\bar{G}_{0}(x)}=\frac{\bar{F}(x)}{x^{-\frac{1}{\epsilon}}}\frac{x^{-\frac{1}{\epsilon}}}{\bar{G}_{0}(x)}=\frac{\int_A f_{\boldsymbol{Z}}({\boldsymbol{z}}) x^{-\frac{1}{\xi_{\boldsymbol{z}}}}L_{\boldsymbol{z}}(x)d{\boldsymbol{z}}}{x^{-\frac{1}{\epsilon}}}\frac{x^{-\frac{1}{\epsilon}}}{\bar{G}_{0}(x)}.
\end{split}
\end{equation}
As before we can prove that the first fraction $\frac{\bar{F}(x)}{x^{-\frac{1}{\epsilon}}}\rightarrow\infty$. The expression $\frac{x^{-\frac{1}{\epsilon}}}{\bar{G}_{0}(x)}$ goes to $\infty$ as well due to Lemma \ref{lemma2}, thus
\begin{equation}\label{l3_14}
\begin{split}
\lim_{x\rightarrow\infty}\frac{\bar{F}(x)}{\bar{G}_{0}(x)}=\infty.
\end{split}
\end{equation}
If $\xi_F$ was $0$, then for some $G_0\in MDA(0)$ we would have 
\begin{equation}\label{l3_15}
\begin{split}
\lim_{x\rightarrow\infty}\frac{\bar{F}(x)}{\bar{G}_{0}(x)}=\lim_{x\rightarrow\infty}1=1,
\end{split}
\end{equation}
hence $\xi_F\neq0$.
\newline
\newline
Finally we prove that, if $\xi_{\boldsymbol{z}}$ is continuous in ${\boldsymbol{z}}$ and $\xi_{\max}$ exists, then we have $\xi_F=\xi_{\max}$.
We will first separate $A$ in two sets $A_1, A_2$, where $A_1=\{{\boldsymbol{z}}|\xi_{\max}-\lambda\leq \xi_{\boldsymbol{z}} \leq \xi_{\max} \}$ and $A_2=\{{\boldsymbol{z}}|\xi_{lo} \leq \xi_{\boldsymbol{z}} < \xi_{\max}-\lambda\}$. Since $\xi_{\boldsymbol{z}}$ is continuous, then the pre-image of each of the measurable sets $[\xi_{\max}-\lambda,\xi_{\max}],[\xi_{lo},\xi_{\max}-\lambda)$ will be measurable. In addition, since $[\xi_{\max}-\lambda,\xi_{\max}]$ and $[\xi_{lo},\xi_{\max}-\lambda)$ contain an open set, then so will $A_1$ and $A_2$, implying that $p_i=\mathbb{P}(A_i)>0$, where $i \in \{1,2\}$. Thus,

\begin{equation}\label{integral_separ}
\begin{split}
\bar{F}(x)=\int_A f_{\boldsymbol{Z}}({\boldsymbol{z}}) \bar{F}_z(x)d{\boldsymbol{z}}=
p_1\int_{A_1} \frac{f_{\boldsymbol{Z}}({\boldsymbol{z}})}{p_1}\bar{F}_{\boldsymbol{z}}(x)d{\boldsymbol{z}}+p_2\int_{A_2} \frac{f_{\boldsymbol{Z}}({\boldsymbol{z}})}{p_2} \bar{F}_{\boldsymbol{z}}(x)d{\boldsymbol{z}} \\=p_1 \bar{F}_1(x)+p_2 \bar{F}_2(x).
\end{split}
\end{equation}
From the first part of the Theorem: $\xi_1\in [\xi_{\max}-\lambda,\xi_{\max}]$, and $\xi_2\in [\xi_{lo},\xi_{\max}-\lambda]$, where $F_i\in MDA(\xi_i),\ i=1, 2$. On the other hand Theorem \ref{otheorem1} implies that $\xi_F=\xi_1 $, therefore $\xi_F \in [\xi_{\max}-\lambda,\xi_{\max}]$ for all $\lambda>0$. We conclude that $\xi_F=\xi_{\max}$.

\subsection*{Proof of Lemma \ref{lemma5}}

We assume that $\xi_F>\epsilon$. Then as in the earlier derivations, due to dominated convergence and Lemmas \ref{lemma1} and \ref{lemma2}, for any $\delta>0$, we get:
\begin{equation}\label{l5_1}
\begin{split} 
\lim_{x\rightarrow\infty} \frac{x^{-\frac{1}{\xi_F}}L_F(x)}{x^{-\frac{1}{\epsilon+\delta}}}=\lim_{x\rightarrow\infty}\frac{\bar{F}(x)}{x^{-\frac{1}{\epsilon+\delta}}}=\lim_{x\rightarrow\infty} \int_A f_{\boldsymbol{Z}}({\boldsymbol{z}})\frac{\bar{F}_{\boldsymbol{z}}(x)}{x^{-\frac{1}{\epsilon+\delta}}}d{\boldsymbol{z}}\\
=\lim_{x\rightarrow\infty} \int_{A^+} f_{\boldsymbol{Z}}({\boldsymbol{z}})\frac{\bar{F}_{\boldsymbol{z}}(x)}{x^{-\frac{1}{\epsilon+\delta}}}d{\boldsymbol{z}}+\lim_{x\rightarrow\infty} \int_{A^-} f_{\boldsymbol{Z}}({\boldsymbol{z}})\frac{\bar{F}_{\boldsymbol{z}}(x)}{x^{-\frac{1}{\epsilon+\delta}}}d{\boldsymbol{z}}\\
=\int_{A^+} \lim_{x\rightarrow\infty}  f_{\boldsymbol{Z}}({\boldsymbol{z}})\frac{x^{-\frac{1}{\xi_{\boldsymbol{z}}}}}{x^{-\frac{1}{\epsilon+\delta}}}L_{\boldsymbol{z}}(x)d{\boldsymbol{z}} + \int_{A^-} \lim_{x\rightarrow\infty}  f_{\boldsymbol{Z}}({\boldsymbol{z}})\frac{\bar{F}_{\boldsymbol{z}}(x)}{x^{-\frac{1}{\epsilon+\delta}}}d{\boldsymbol{z}}=0.\\
\end{split}
\end{equation}
therefore $\xi_F<\epsilon+\delta, \forall \delta>0$, contradicting our assumption $\xi_F>\epsilon$.

\subsection*{Proof of Theorem \ref{otheorem2}}

The proof is similar to that of the last statement in Theorem \ref{otheorem2a}.
We will first separate $A$ in two sets $A_1, A_2$, where $A_1=\{{\boldsymbol{z}}|\xi_{\max}-\lambda\leq \xi_{\boldsymbol{z}} \leq \xi_{\max} \}$ and $A_2=\{{\boldsymbol{z}}| \xi_{\boldsymbol{z}} < \xi_{\max}-\lambda\}$. Since $\xi_{\boldsymbol{z}}$ is continuous, then the pre-image of each of the measurable sets $[\xi_{\max}-\lambda,\xi_{\max}],(-\infty,\xi_{\max}-\lambda)$, will be measurable. In addition, since $[\xi_{\max}-\lambda,\xi_{\max}]$ and $(-\infty,\xi_{\max}-\lambda)$ contain an open set, then so will $A_1$ and $A_2$, implying that $p_i=\mathbb{P}(A_i)>0$, where $i \in \{1,2\}$.

\begin{equation}\label{c2_1}
\begin{split}
\bar{F}(x)=\int_A f_{\boldsymbol{Z}}({\boldsymbol{z}}) \bar{F}_z(x)d{\boldsymbol{z}}=
p_1\int_{A_1} \frac{f_{\boldsymbol{Z}}({\boldsymbol{z}})}{p_1}\bar{F}_{\boldsymbol{z}}(x)d{\boldsymbol{z}}+p_2\int_{A_2} \frac{f_{\boldsymbol{Z}}({\boldsymbol{z}})}{p_2} \bar{F}_{\boldsymbol{z}}(x)d{\boldsymbol{z}} \\=p_1 \bar{F}_1(x)+p_2 \bar{F}_2(x).
\end{split}
\end{equation}
Based on Theorem \ref{otheorem2a} and Lemma \ref{lemma5}: $\xi_1=\xi_{\max}$, and $\xi_2\in (-\infty,\xi_{\max}-\lambda]$, where $F_i\in MDA(\xi_i),\ i=1, 2$. From Theorem \ref{otheorem1}, we conclude that $\xi_F=\xi_{\max}$.
\
The last statement in the Theorem, that is, if  $\xi_{\max} \leq0$  then $\xi_F \leq0$, is simply Corollary \ref{corollarynonpositive}.

\subsection*{Proof of Proposition \ref{proposition2}}

In the case that $\xi_X>0$, based on our assumptions there exists $L(x)$ such that 
\begin{equation}\label{p2_1}
\mathbb{P}(X>x)=\bar{F}_X(x)=x^{-\frac{1}{\xi_X}}L_1(x).
\end{equation}
Therefore
\begin{equation}\label{p2_2}
\bar{F}_Y(x)=\mathbb{P}(Y>x)=\mathbb{P}(X^\alpha>x)=\mathbb{P}(X>x^{\frac{1}{\alpha}})=({x^{\frac{1}{\alpha}}})^{-\frac{1}{\xi_X}}L_1(x^\frac{1}{\alpha})=x^{-\frac{1}{\alpha\xi_X}}L_2(x).
\end{equation}
We conclude that $Y \in MDA(\alpha\xi_X)$.
On the other hand if $\xi_X\leq 0$ then $\xi_Y\leq 0$, because if $\xi_Y>0$, then from the first part we would have $\xi_X=\frac{1}{\alpha}\xi_Y>0$.

\subsection*{Proof of Proposition \ref{proposition3}}
We will first prove the case when $p=1$. If we fix $y$ and denote with $\xi_y^{h-}, \xi_y^{h+}$ the shape parameters of the left and right tail of $p(\hat{f}_{\boldsymbol{V}}({\boldsymbol{X}})\big|y)$, then assuming that at least one of them is positive, from Proposition 21 we know that the tail shape parameter of $p(|\hat{f}_{\boldsymbol{V}}({\boldsymbol{X}})|\big|y)$ is $\xi_y^{h}=\max\{\xi_y^{h-}, \xi_y^{h+}\}$. We notice now that $\xi_y^{h-}, \xi_y^{h+}$ are the right and left tail shape parameters of $p(-\hat{f}_{\boldsymbol{V}}({\boldsymbol{X}})\big|y)$, therefore they are the right and left tail shape parameters of the distribution $p(y-\hat{f}_{\boldsymbol{V}}({\boldsymbol{X}})\big|y)$. Due to this, if we denote with $\xi_y^g$ the tail shape parameter of $p(|y-\hat{f}_{\boldsymbol{V}}({\boldsymbol{X}})|\big|y)$, using Proposition 21 once again we have that $\xi_y^g=\max\{\xi_y^{g+}, \xi_y^{g-}\}=\max\{\xi_y^{h-}, \xi_y^{h+}\}=\xi_y^{h}$, where $\xi_y^{g-}, \xi_y^{g+}$ are the left and right shape parameters of $p(y-\hat{f}_{\boldsymbol{V}}({\boldsymbol{X}})|y)$. If both $\xi_y^{h-}, \xi_y^{h+}$ are non-positive then from Proposition 21, $\xi_y^h$ is non-positive, and furthermore $\xi_y^g$ is non-positive, otherwise we could go in the reverse direction and prove that $\xi_y^g>0$ implies that either $\xi_y^{g-}=\xi_y^{h+}$ is positive, or that $\xi_y^{g+}=\xi_y^{h-}$ is positive. 
\newline
Now, we denote by $G_y(s)$ the distribution of $|y-\hat{f}_{\boldsymbol{V}}({\boldsymbol{X}})|$ given $y$, and prove that the family $\{G_y(s)| y \in S\}$ has stable cross-tail variability. For each $y$ we denote with $t_0(y)$ the smallest value after which the sub-polynomial assumption is satisfied by $F_y(t)$. Similarly we define $s_0(y)$ for $G_y(s)$. Since the family $\{F_y(t)| y \in S\}$ has stable cross-tail variability, then each such $t_0(y)$ exists, and furthermore the set $\{t_0(y)| y \in S\}$ is bounded from above. Since each $s_0(y)$ is only displaced by a magnitude of $|y|$ from $t_0(y)$, and since the set $S$ is bounded, then we can conclude that $\{s_0(y)| y \in S\}$ is bounded from above.
\newline
We denote $\xi^g, \xi^h$ the tail shape parameters  of $|Y-\hat{f}_{\boldsymbol{V}}({\boldsymbol{X}})|$ and $|\hat{f}_{\boldsymbol{V}}({\boldsymbol{X}})|$ respectively. Using Theorem \ref{otheorem2} twice we get that if there is at least one $\xi_y^h=\xi_y^g>0$ then $\xi^h=\max\{\xi_y^h|y\in S\}=\max\{\xi_y^g|y\in S\}=\xi^g>0$, otherwise $\xi^h\leq 0, \xi^g\leq 0$.
\newline
Finally we finish the proof by applying Proposition 22 on $|Y-\hat{f}_{\boldsymbol{V}}({\boldsymbol{X}})|$ and $|\hat{f}_{\boldsymbol{V}}({\boldsymbol{X}})|$. 

\section*{Appendix B: Examples where the regularity conditions do not hold}

Below we give examples where the regularity conditions do not hold: 
\newline
\newline
\textbf{Example 1}: Let $f_U(u)$ be a uniform distribution, and $g_u(w)$ an exponential distribution with parameter $\frac{1}{u}$. Clearly, the expectation of $g_u(w)$ at each $u \in (0,1)$ exists. However for 
\newline
\begin{equation}\label{eg1_1}
h(w)=\int_{0}^{1}f_U(u) g_u(w)du= \int_{0}^{1}ue^{-uw}du
\end{equation} 
the expectation is
\begin{equation}\label{eg1_2}
\int_{0}^{\infty}\int_{0}^{1}wf_U(u) g_u(w)du dw=\int_{0}^{1}\int_{0}^{\infty}wue^{-uw}dwdu=\int_{0}^{1}\frac{1}{u}du
\end{equation} 
In this example, we can see that even though all the distributions $g_u(w)$ have shape parameter $0$, the shape parameter of $h(w)$ is bigger or equal to one. This is because the beginning of the exponential behaviour of the tail is delayed indefinitely across the elements of the family, violating the $\gamma$-uniform sub-polynomial assumption. 
\newline
\newline
Below we give an example of a family of slowly-varying functions  $\{L_z(x)| z\in A\}$, where $A$ is compact and $L_z(x)$ is continuous in $x$ and $z$, but $\{L_z(x)| z\in A\}$ is not $\gamma$-uniformly sub-polynomial. In this case, the non slowly-varying behaviour (non sub-polinomiality) of $L_z(x)$, or in other words, the tail of $F_z(x)$, is postponed indefinitely across the family of $\{F_z(x)|z \in A\}$
\newline
\newline
\textbf{Example 2}: Let $L_z(x)$, for $z \in [0,1]$, be defined as below:
\begin{align}\label{eg2_1}
    L_z(x)=
    \begin{cases}
     1 + zx^{4-(z-\frac{1}{x})^2} & \text{for $x \in (1,\frac{1}{z})$}
      \\
     1 + \frac{1}{z^3}  & \text{for $x \in (\frac{1}{z},\infty)$}
    \end{cases}  
\end{align} when $z\neq0$ and $L_0=1 \text{ for } x \in (\frac{1}{z},\infty)$.
For $x^{-1}$ we define $F_z(x)=x^{-1}L_z(x)$, that is:
\begin{align}\label{eg2_21}
    F_z(x)=
    \begin{cases}
     x^{-1} + zx^{3-(z-\frac{1}{x})^2} & \text{for $x \in (1,\frac{1}{z})$}
      \\
     x^{-1} + \frac{1}{z^3}x^{-1}  & \text{for $x \in (\frac{1}{z},\infty)$}
    \end{cases}  
\end{align} when $z\neq0$ and $F_0=x^{-1} \text{ for } x \in (\frac{1}{z},\infty)$. One can check that $F_z(x)$ and $L_z(x)$ are continuous in $z$. On the other hand for a given $z$, $F_z(\frac{1}{z})=z+z^{-2}$, meaning that $F_z(\frac{1}{z})$ tends to infinity, when $z$ tends to zero. Therefore $\{L_z(x)| z\in A\}$ is not $\gamma$-uniformly sub-polynomial.

\section*{Appendix C: Examples where the regularity conditions hold}

Below we give examples where the regularity conditions do hold: 
\newline
\newline
\textbf{Example 3}: Let $\bar{F}_z(x)=x^{-z}=x^{-\frac{1}{\frac{1}{z}=\xi_z}}$ for $z \in (1,\infty)$, and let $\bar{F}(x)=e\int_1^{\infty} e^{-z}\bar{F}_z(x) dz$. Then $\bar{F}(x)=x^{-1}\frac{1}{1+\ln{x}}=x^{-1}L(x)$, where $L(x)=\frac{1}{\ln{x}}$ is slowly varying as both $1$ and $\ln{x}$ are slowly varying.
\newline
\newline
\textbf{Example 4}: Let $\bar{F}_z(x)=x^{-z}\ln{x^z}$ for $z \in (1,2)$, and let $\bar{F}(x)=\int_1^{2} \bar{F}_z(x) dz$. Then $\bar{F}(x)=x^{-1}-2x^{-2}+x^{-1}\frac{1}{\ln{x}}-x^{-2}\frac{1}{\ln{x}}=x^{-1}(1-2x^{-1}+\frac{1}{\ln{x}}-x^{-1}\frac{1}{\ln{x}})=x^{-1}L(x)$, where $L(x)=1-2x^{-1}+\frac{1}{\ln{x}}-x^{-1}\frac{1}{\ln{x}}$ is slowly varying.

\section*{Appendix D: Moment based motivation}

In Proposition \ref{proposition3}, we showed that under certain conditions, we could estimate the shape of the tail of the distribution of $W_{{\boldsymbol{V}}}({\boldsymbol{U}})$  without using test labels. This can also be motivated from the moments of $W_{{\boldsymbol{V}}}({\boldsymbol{U}})$. Indeed, conditioning on the test label $y$ we have
\begin{equation}\label{4.2.7}
\mathbb{E}[W^{p}_{{\boldsymbol{V}}}({\boldsymbol{U}})|Y=y]=E_{\boldsymbol{V}}[({y}-\hat{f}_{\boldsymbol{V}}(\boldsymbol{x}))^{p}|y]
\end{equation}
\begin{equation}\label{4.2.8}
=\sum\limits_{k=0}^{p} {\binom{p}{k}}y^k (-1)^{p-k}E_{\boldsymbol{V}}[\hat{f}_{\boldsymbol{V}}^{p-k}(\boldsymbol{x})|y] 
\end{equation}

We can see that for test label $y$, if the moment $p$ of $\hat{f}_{\boldsymbol{V}}(\boldsymbol{x})$ given $y$ exists then the moment $p$ of $W_{{\boldsymbol{V}}}(u)$ given $y$ exists. If each $E_{\boldsymbol{V}}[\hat{f}_{\boldsymbol{V}}^j(\boldsymbol{x})|y]$, $j\in\{1,...,p\}$ changes continuously with $y$ then $\mathbb{E}[W^{p}_{{\boldsymbol{V}}}({\boldsymbol{U}})|y]$ is continuous with respect to $y$. Further assuming that the support of $Y$ is compact, then moment $p$ of $W_{{\boldsymbol{V}}}({\boldsymbol{U}})$, that is, $\mathbb{E}[W^{p}_{{\boldsymbol{V}}}({\boldsymbol{U}})]=\mathbb{E}_y\mathbb{E}[W^{p}_{{\boldsymbol{V}}} ({\boldsymbol{U}})|Y=y]$ will exist as well. 
\newline
\newline
Under these conditions, if $\hat{f}_{\boldsymbol{V}}(\boldsymbol{x})$ is a non-negative function, then the existence of $\mathbb{E}[\hat{f}_{\boldsymbol{V}}^{p}(\boldsymbol{x})]=\mathbb{E}_y\mathbb{E}[\hat{f}_{\boldsymbol{V}}^{p}(\boldsymbol{x})|y]$ guarantees the existence of $\mathbb{E}[\hat{f}_{\boldsymbol{V}}^{p}(\boldsymbol{x})|y]$ for almost all $y$, thus it ensures the existence of $\mathbb{E}[W^{p}_{{\boldsymbol{V}}}({\boldsymbol{U}})]$.

\section*{Appendix E: Reducing the variability of the estimated shape parameters}\label{reducvar}
 It is proven in \cite{10.1214/aos/1176347396}, that under certain conditions on $k$ (in particular that $\frac{k(n)}{n}\rightarrow0$ as $n\rightarrow\infty$) the Pickands Estimator has an asymptotically Gaussian distribution: $\sqrt{k(n)}(\hat{\xi}^{(P)}_{k,n}-\xi)\xrightarrow{d}\mathcal{N}(0, \sigma^2(\xi))$. This implies that for large $n$, we roughly have $\hat{\xi}^{(P)}_{k,n}\sim \mathcal{N}(\xi, \frac{\sigma^2(\xi)}{k(n)})$. Minding the size of $n$, we can split the $n$ samples into $m$ groups such that $n=m\frac{n}{m}$, and such that we still have roughly $\hat{\xi}^{(P)}_{k,\frac{n}{m}}\sim \mathcal{N}(\xi, \frac{\sigma^2(\xi)}{k(\frac{n}{m})})$. Since we can estimate $\hat{\xi}^{(P)}_{k,\frac{n}{m}}$ for each of the $m$ groups we can define the average estimation as $\hat{\xi}^{(P),avg}_{k,\frac{n}{m}}=\frac{1}{m}\sum_{i=1}^m \hat{\xi}^{(P),i}_{k,\frac{n}{m}}$. Under the assumption that samples from such groups are independent, we get that $\hat{\xi}^{(P),avg}_{k,\frac{n}{m}}\sim \mathcal{N}(\xi, \frac{\sigma^2(\xi)}{m k(\frac{n}{m})})$. Since $k(n)=o(n)$, we can choose to reduce the variance 'linearly' by keeping $\frac{n}{m}$ constant and increasing $m$, instead of increasing the sub-linear $k(n)$. This becomes quite apparent if we set $k(n)=\log{n}$ or $k(n)=\sqrt{n}$. Indeed, for $k(n)=\log{n}$, the ratio between the variances of the direct approach and our approach is 

 \begin{equation}
     \frac{m \log{\frac{n}{m}}}{\log{n}}=\frac{m \log{\frac{n}{m}}}{\log{m}+\log{\frac{n}{m}}}=\frac{m C}{\log{m}+C}\rightarrow\infty 
 \end{equation}
as $m\rightarrow\infty$.

Similarly for $k(n)=\sqrt{n}$,
 \begin{equation}
     \frac{m \sqrt{\frac{n}{m}}}{\sqrt{n}}=\sqrt{m}\rightarrow\infty 
 \end{equation}
 as $m\rightarrow\infty$. Here we can see that even if we fix $m$ and then allow each group with size $\frac{n}{m}$ to grow as $n$ increases, the variance is still $\sqrt{m}$ times smaller using our approach. 

The asymptotically Gaussian distribution property holds in the case of the DEdH estimator if one knows that $\xi>0$ (Hill estimator, \cite{10.1214/aos/1176346804}). Furthermore, both estimators $H^{(1)}_{k,n}$ and $H^{(2)}_{k,n}$ in Definition \ref{dedh_estimator_def} jointly possess this property, \cite{10.1214/aos/1176347397}.

\section*{Appendix F: The inadequacy of the direct POT usage on mixture distributions}

In this section, we illustrate two cases where cross tail estimation is necessary for proper tail shape estimation.
\begin{figure}[ht]
\begin{tabular}{ll}
\includegraphics[scale=0.18]{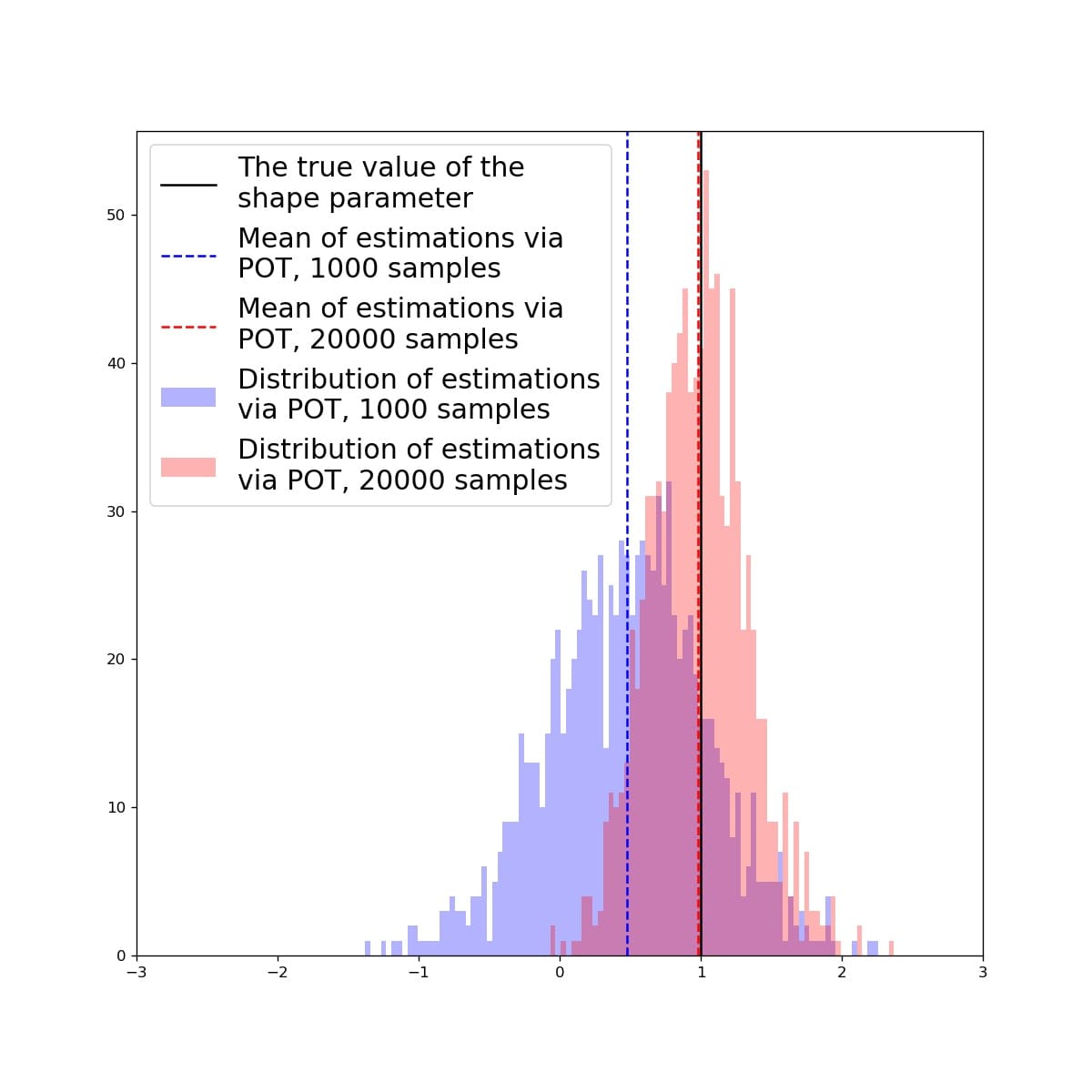}
&
\includegraphics[scale=0.18]{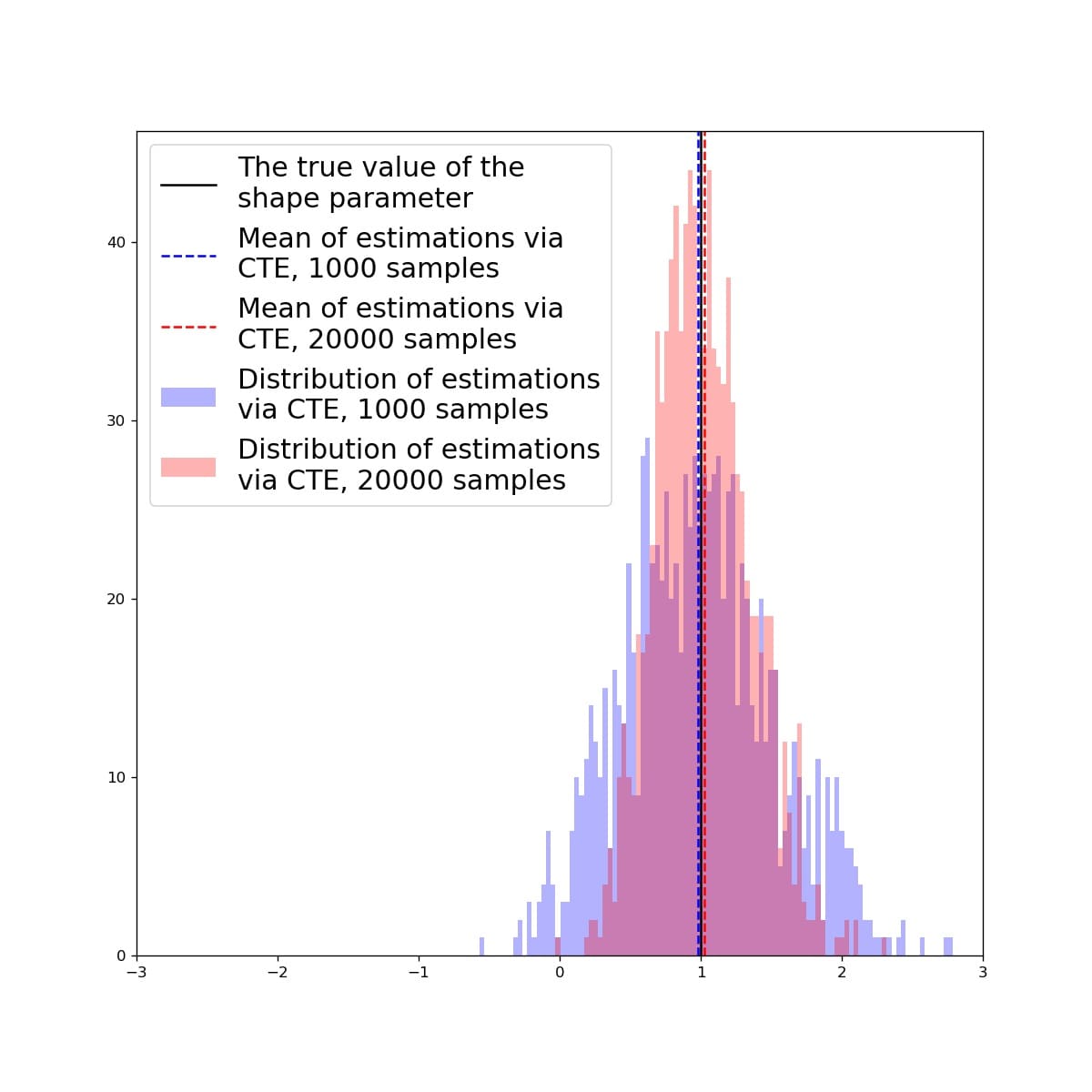}
\end{tabular}
\caption{Standard estimation of the shape parameter of the tails by simply applying the Pickands' Estimator, on average, gives poor results on fewer data (left). Cross tail estimation (CTE) gives the correct estimation on average. (right).}
\label{fig4.2.1}
\end{figure}
\subsection*{Uniform Case}

In our experimental procedure, we randomly select samples adhering to two distinct power law distributions. Each of these distributions has a unique characteristic shape parameter - one has a shape parameter of $1$, while the other possesses a shape parameter of $0.5$. For our random sampling process, we afford equal probability, precisely $50\%$, to both these distributions. This means there is an identical chance of picking a sample from either of these power law distributions, each with their respective shape parameters. 

When we examine an experimental set of $10^3$ sampled points from each of these distributions, the resulting pattern becomes apparent as shown in Figure \ref{fig4.2.1} (left). We find that if we amalgamate all the sampled data points from both distributions into a unified array, and subsequently apply Pickands Estimator on this consolidated data set, the process yields a sub-optimal estimation of the distribution tail. The outcome is unsatisfactory as it fails to reveal the accurate shape of the tail, thereby defeating the purpose of the estimation.

However, we discover that there is a noticeable enhancement in the quality of the estimation when we bolster the sample size from the initial $10^3$ to a considerably larger size of $2*10^4$. This increase in sample size permits us to retrieve the true shape of the distribution tail. 

 Using CTE however, we find that a sample size of just $10^3$ proves to be adequate in obtaining a satisfactory estimation of the distribution tail. As illustrated in Figure \ref{fig4.2.1} (right), this method leads to an accurate estimation with a substantially smaller sample size. Therefore, our method introduces an efficient pathway towards achieving accurate estimations with fewer resources, thereby demonstrating its potential superiority over the traditional Pickands Estimator.

\subsection*{Non-Uniform Case}

Similarly, in the second experiment, we sample with $ 20\%$ probability from a distribution with power law tails with shape parameter $1$, and with $80\%$ probability probability from a distribution with power law tails with shape parameter $0.5$. 
\begin{figure}[ht]
\begin{tabular}{ll}
\includegraphics[scale=0.18]{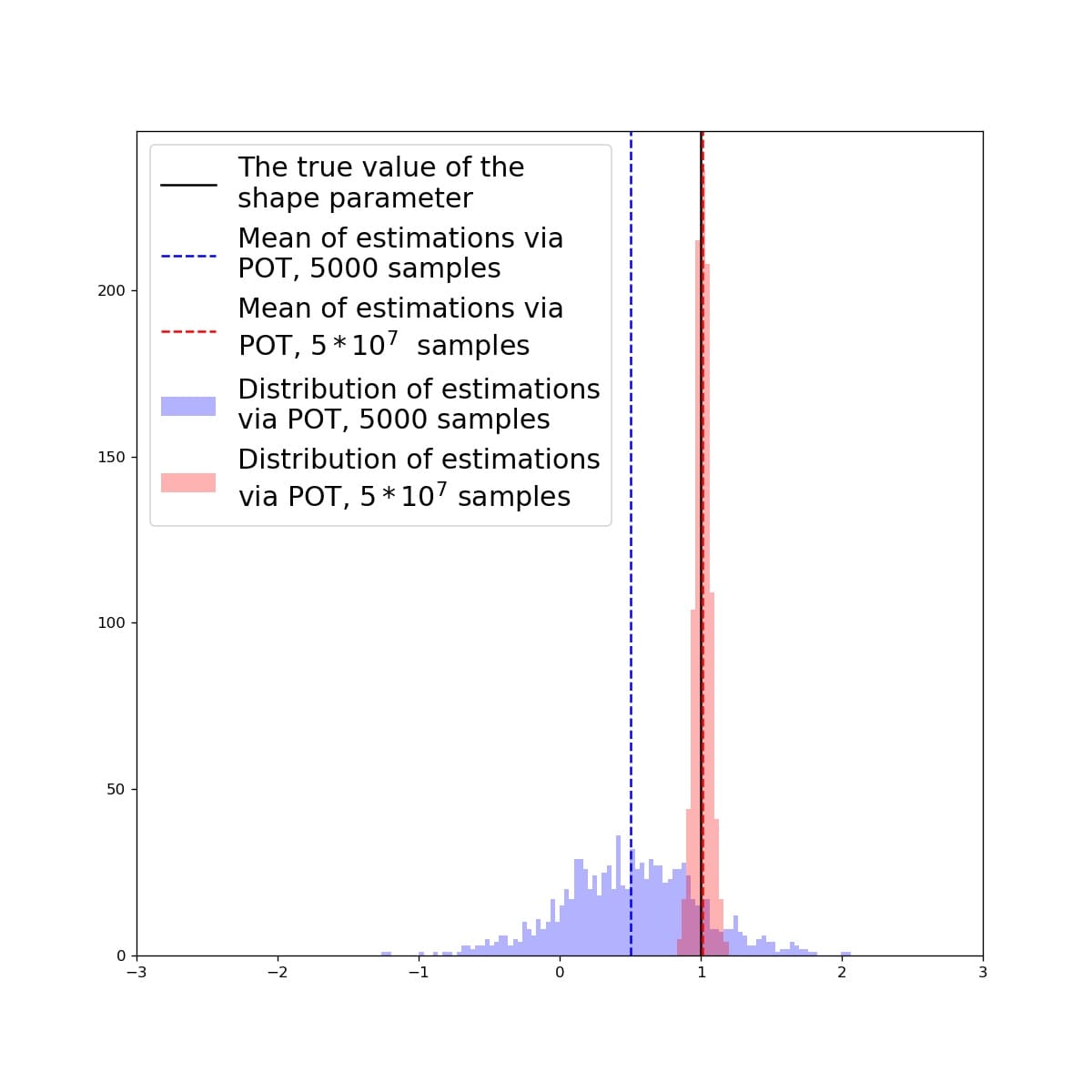}
&
\includegraphics[scale=0.18]{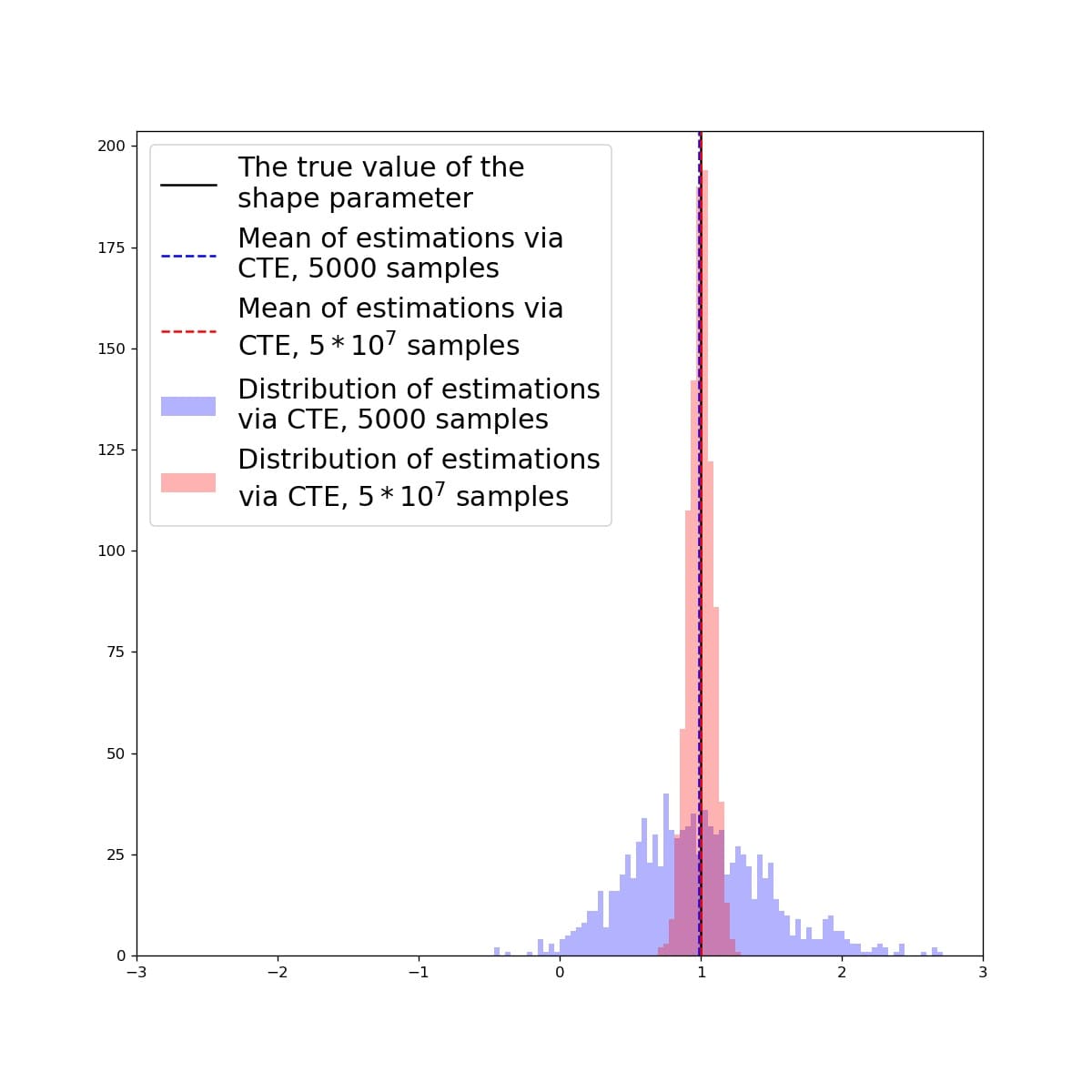}
\end{tabular}
\caption{ Standard estimation of the shape parameter of the tails by simply applying the Pickands' Estimator, on average, gives poor results on fewer data (left). Cross tail estimation (CTE) gives the correct estimation on average. (right).}
\label{fig4.2.2}
\end{figure}
 
When sampling $5*10^3$ points from each distribution, Figure \ref{fig4.2.2}, we are not able to properly estimate the tail if we join all the samples together in a common array and then apply the Pickands' Estimator. But, if we increase the sample size from $5*10^3$ to $5*10^7$, we manage to retrieve the the true tail shape of the mixture.
However, using our method, $5*10^3$ samples are already sufficient to get a proper estimation.

\section*{Appendix G: Additional details with regards to Section 5.1}\label{experiment1_details}

Below we provide Figure \ref{evolution_ksi} which illustrates how ${\xi_z}$ evolves depending on the $\xi_{\max}$ which is given as input. The parameter $\xi_{\max}$ takes the following $45$ values $\{-4,-4+0.1,-4+0,2,...,5\}$.

\begin{figure}[H]
\centering
\includegraphics[scale=0.22]{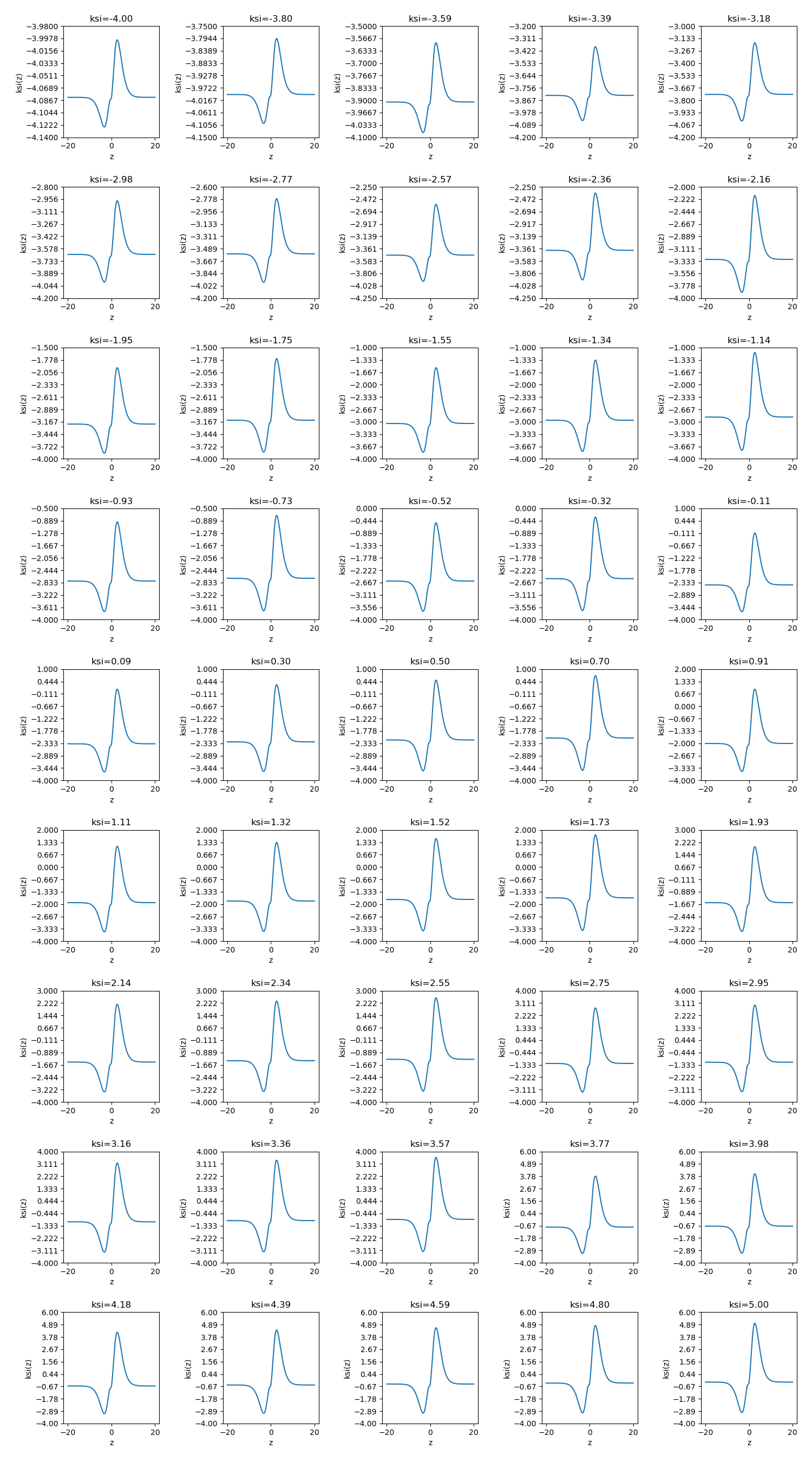}
\caption{The evolution of ${\xi_z}$ depending on the value of $\xi_{\max}$.}
\label{evolution_ksi}
\end{figure}  

\section*{Appendix H: Additional details with regards to Section 5.3}\label{experiment3_details}

\begin{figure}[H]
\centering
\includegraphics[scale=0.135]{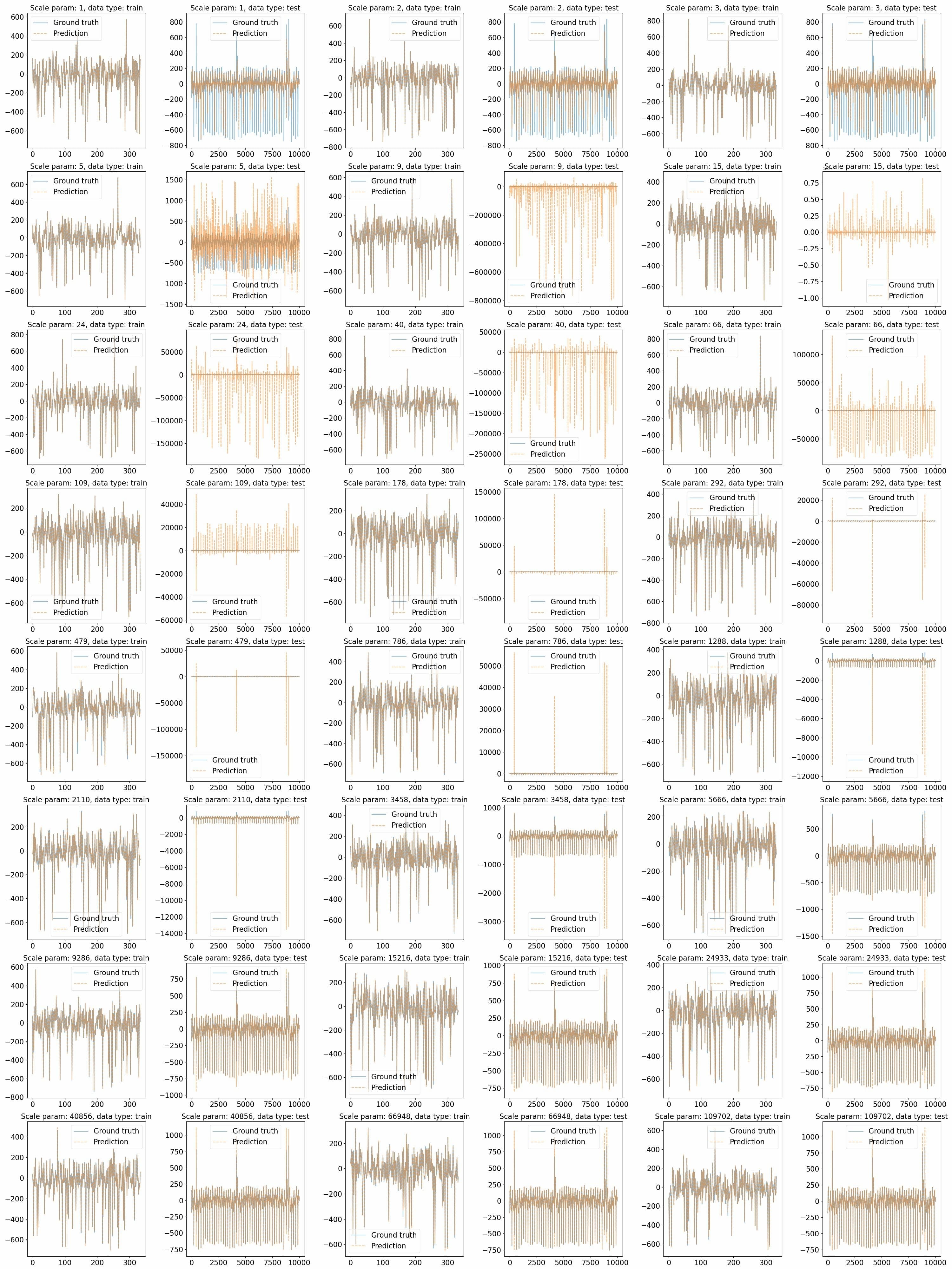}
\caption{The performance of Gaussian process on train and test data depending on the length scale parameter. First half of the cases.}
\label{gridgauss1}
\end{figure}  

\begin{figure}[H]
\centering
\includegraphics[scale=0.14]{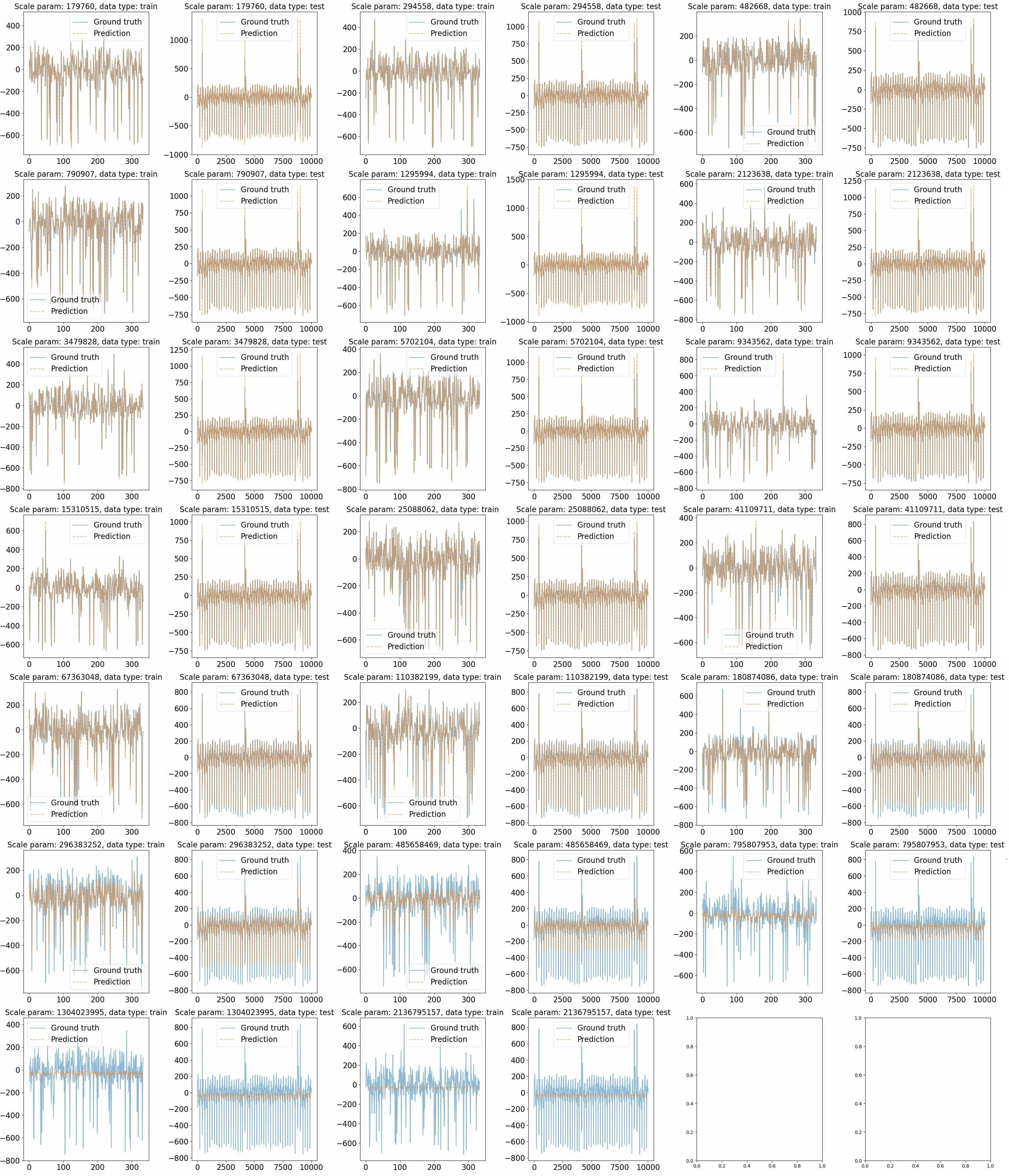}
\caption{The performance of Gaussian process on train and test data depending on the length scale parameter. Second half of the cases.}
\label{gridgauss2}
\end{figure}  

\begin{figure}[H]
\centering
\includegraphics[scale=0.18]{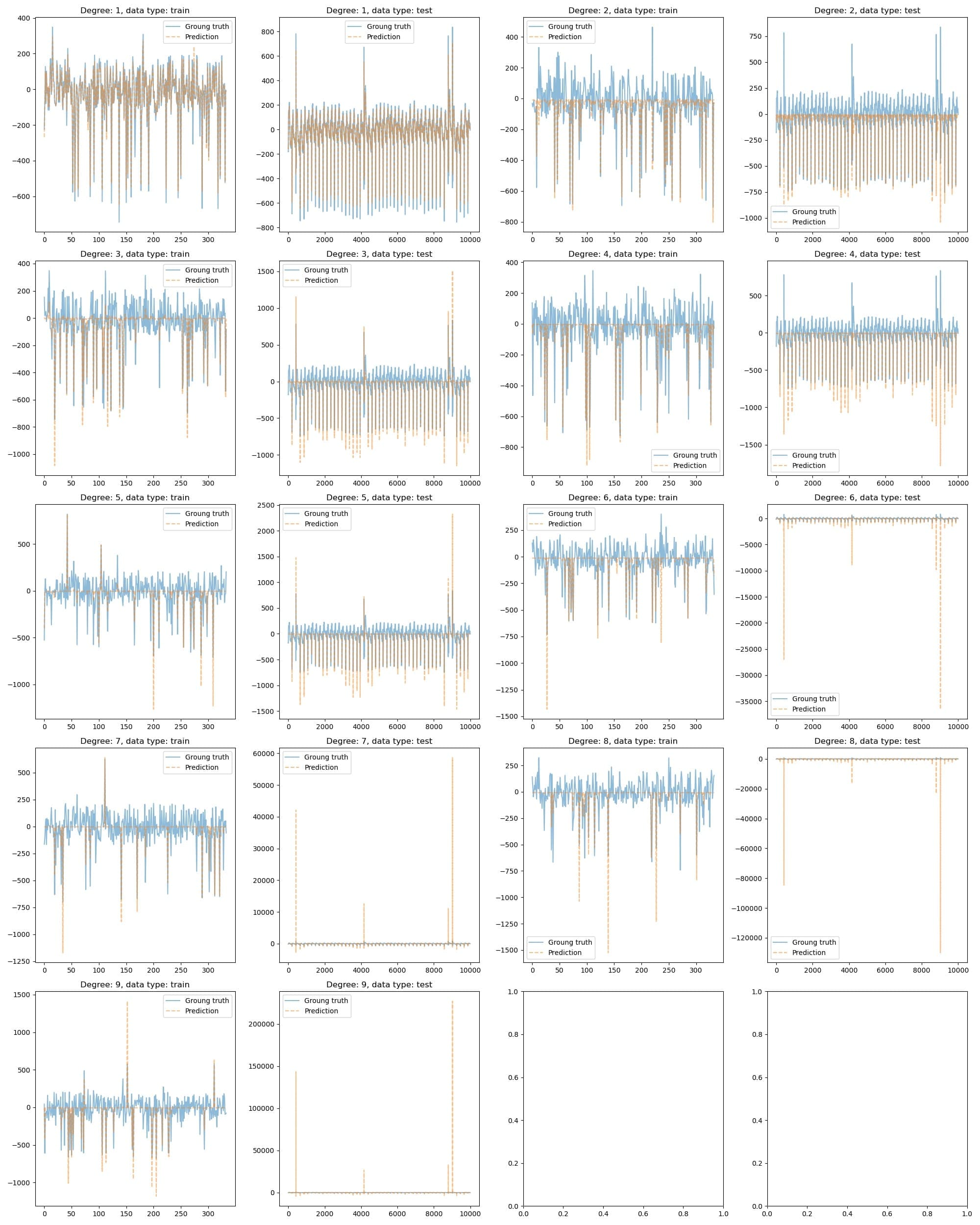}
\caption{The performance of polynomial kernels on train and test data depending on the degree.}
\label{gridpoly}
\end{figure}  

\begin{figure}[H]
\centering
\includegraphics[scale=0.125]{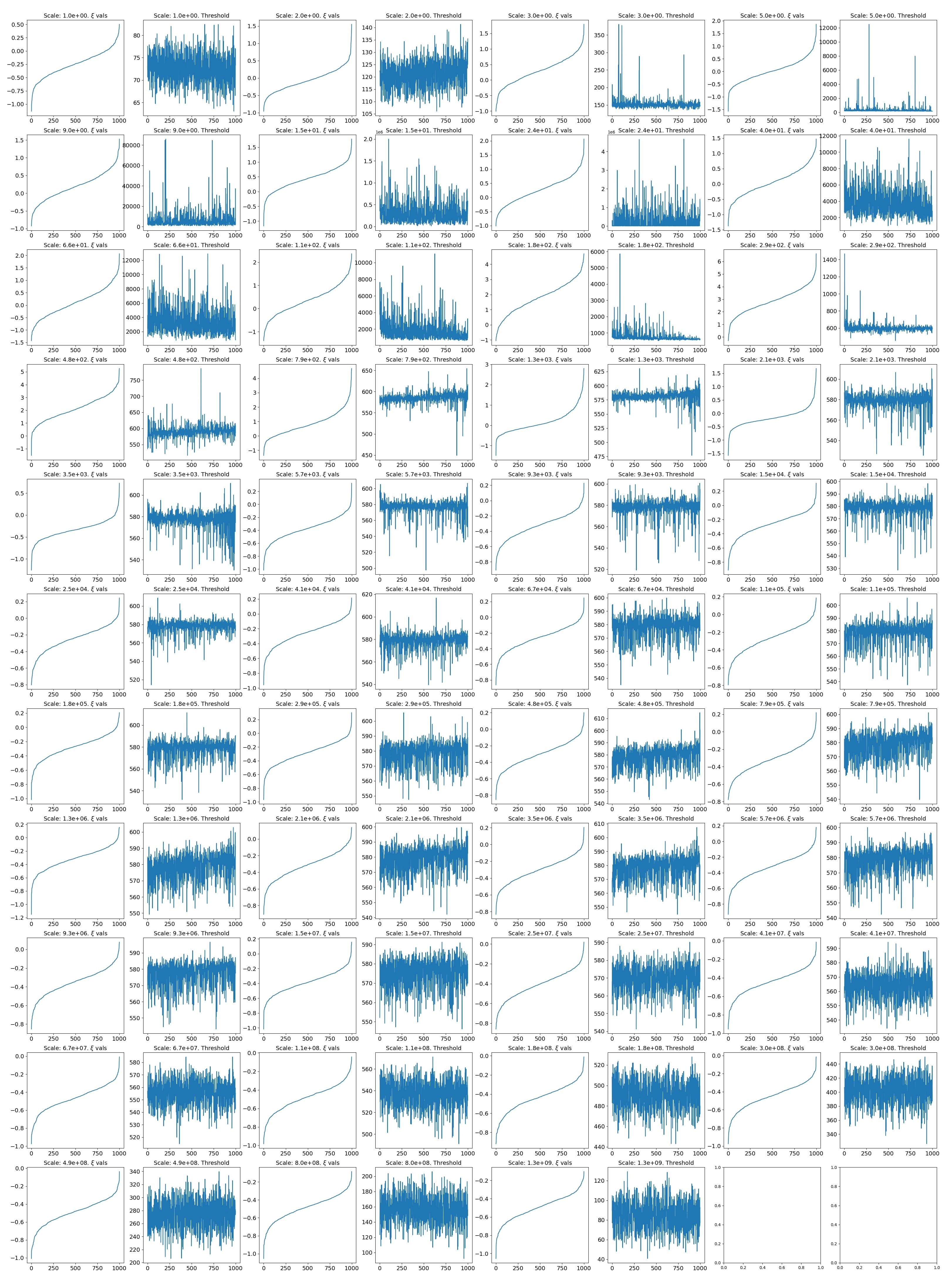}
\caption{For each length scale parameter of the Gaussian Process, we present the variability (sorted) of the estimated shape parameters across 1000 conditional distributions (defined by the choice of training sets). Jointly, we also present the 97th percentile of the conditional distributions corresponding to each estimated shape parameter. }
\label{variabilitygauss}
\end{figure}  

\begin{figure}[H]
\centering
\includegraphics[scale=0.18]{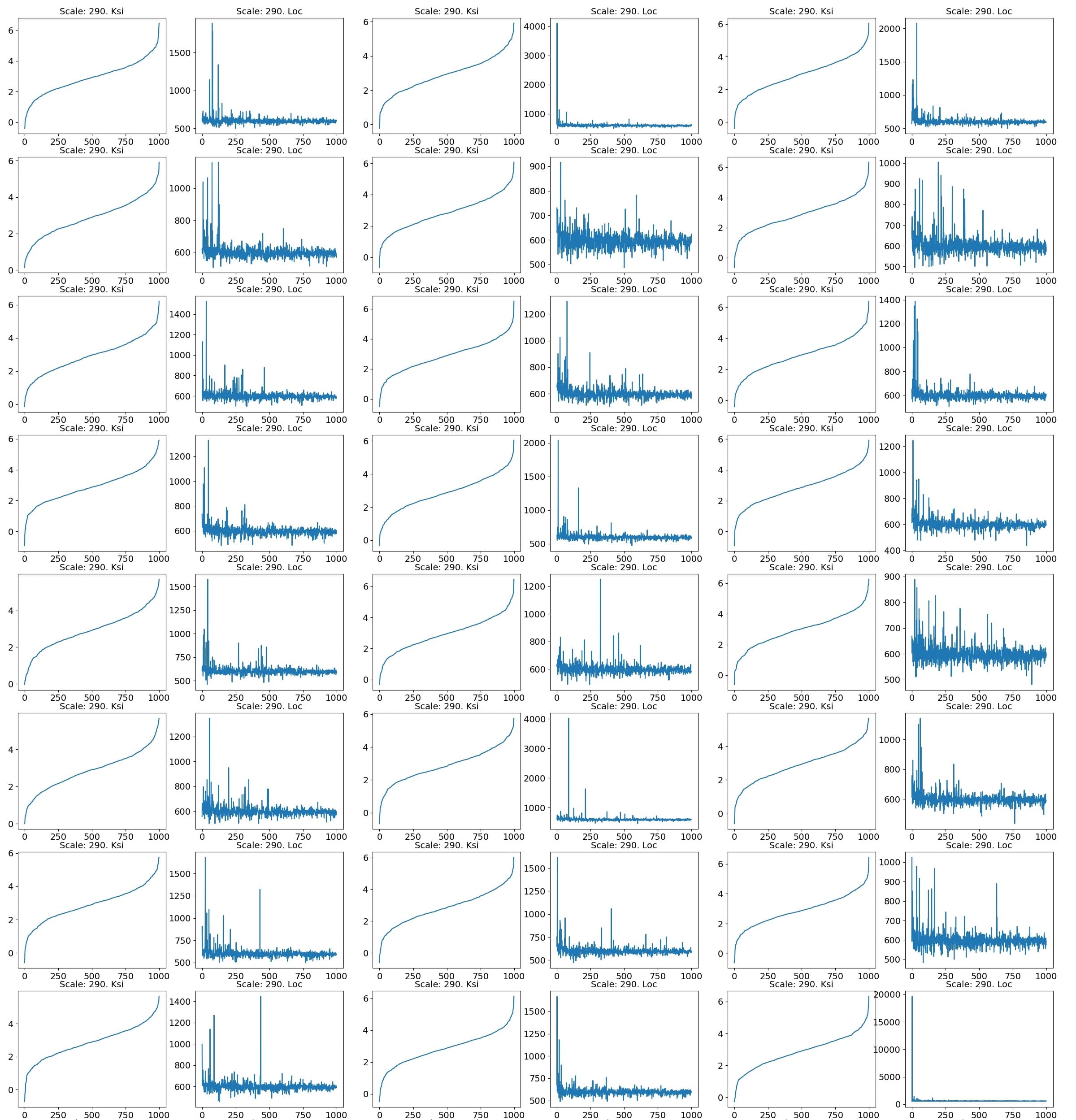}
\caption{We run 32 times the Gaussian Process experiment for length scale parameter value of 290. On each run, we calculate thresholds (sorted) of the 1000 conditional distributions determined by the 1000 choices of the training set, as well as their corresponding shape tail parameters. We see that higher thresholds correspond to lower shape parameters}
\label{290gauss}
\end{figure}  

\begin{figure}[H]
\centering
\includegraphics[scale=0.34]{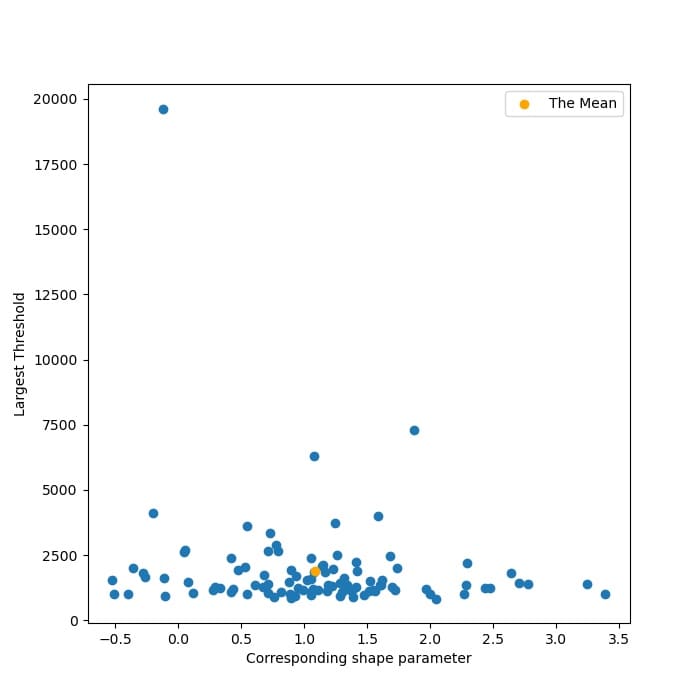}
\caption{For each run, in the experiment of Figure \ref{290gauss} the maximum threshold and the corresponding shape parameter are selected. Here we present the scatter plot of 100 such points, received from 100 runs. We see that high thresholds correspond to lower shape parameters.}
\label{290gaussscatter}
\end{figure}  

\bibliography{main}
\newpage

\end{document}